\documentclass{article} 
\usepackage{iclr2027_conference,times}

\usepackage{hyperref}
\usepackage{url}

\usepackage{amsmath}
\usepackage{amssymb}
\usepackage{amsthm}
\usepackage{thmtools}
\usepackage{booktabs} 

\usepackage{pgfplots}
\usepgfplotslibrary{groupplots}
\pgfplotsset{compat=1.18}
\usepgfplotslibrary{colormaps}

\newtheoremstyle{myplain}
  {6pt}{0pt}
  {\itshape}{}{\bfseries}{.}{.5em}{}
\theoremstyle{myplain}
\newtheorem{lemma}{Lemma}
\newtheorem{theorem}[lemma]{Theorem}
\newtheorem{corollary}[lemma]{Corollary}
\newtheorem{claim}[lemma]{Claim}
\newtheorem{conjecture}[lemma]{Conjecture}

\newtheoremstyle{mydefinition}
  {6pt}{0pt}
  {\normalfont}{}{\bfseries}{.}{.5em}{}
\theoremstyle{mydefinition}

\newtheoremstyle{myremark}
  {6pt}{0pt}
  {\normalfont}{}{\itshape}{.}{.5em}{}
\theoremstyle{myremark}

\renewcommand{\exp}{\mathrm{exp}}

\newcommand{\gnn}{\ensuremath{\mathcal{N}}}

\newcommand{\msl}{\ensuremath{\{\!\{}}
\newcommand{\msr}{\ensuremath{\}\!\}}}

\makeatletter
\renewenvironment{proof}[1][\proofname]{%
  \par
  \pushQED{\qed}%
  \normalfont
  \topsep0pt
  \partopsep0pt
  \trivlist
  \item[\hskip\labelsep\itshape #1\@addpunct{.}]\ignorespaces
}{%
  \popQED
  \endtrivlist
  \@endpefalse
}
\makeatother

\title{Constant-Depth
Rational-Weight GNNs
and their Connection to 
Schanuel's Conjecture}

\author{Daumantas Kojelis \\
Institute of Computer Science \\
Leipzig University \\
Leipzig, Germany \\
\texttt{daumantas.kojelis@uni-leipzig.de}
\And
Carsten Lutz \\
Institute of Computer Science \\
Leipzig University \\
and ScaDS.AI Dresden/Leipzig \\
Leipzig, Germany \\
\texttt{carsten.lutz@uni-leipzig.de}
}

\iclrfinalcopy 
\begin{document}

\maketitle
\fancyhead[L]{}

\begin{abstract}
    A landmark result states that the expressive power of graph neural networks (GNNs) is bounded from above by the 1-dimensional Weisfeiler-Leman (1-WL) test. 
    We ask whether a constant-depth GNN can match the
expressive power of 1-WL while using only  rational
weights. We prove that this is the case when Schanuel's conjecture, a famous open problem from transcendental number theory, is taken to be true. Our GNN
    construction works even with integer weights, and more generally with a broad range of rational weights.
    As the activation function, it may use logistic sigmoid, hyperbolic tangent, or related functions. It is uniform across all graph sizes. 
    We also prove a partial 
    converse: if a GNN with exponentiation, hyperbolic sine or hyperbolic cosine as activation 
    matches the expressive power of the 1-WL test in a slightly more expressive setting where
    graphs may contain certain auxiliary vertices, then this implies Schanuel's conjecture restricted to the numbers that arise in computations of the GNN.
    We also conduct experiments showing that the constructed GNNs  indeed match the 1-WL test.
\end{abstract}

\section{Introduction}

A fundamental result in the theory of graph neural networks (GNNs) is
that their expressive power is bounded from above by the 1-dimensional
Weisfeiler--Leman (1-WL) test, in a version
also known as color refinement; see \cite{xu2018howPowerful,WLGoesNeural}. In particular, if two vertices, or two
graphs, cannot be distinguished by 1-WL, then they cannot be
distinguished by a standard message-passing GNN either. 
Apart from motivating a lot of research on how to break the 1-WL barrier,
this raises
the converse question: to what extent can GNNs actually attain the
1-WL upper bound? This question has been answered in several different
ways.

It was shown in \cite{xu2018howPowerful} that GNNs of constant depth $L$ match
the expressive power of $L$-round 1-WL when the aggregation and
combination functions, as well as the graph-level readout, are
injective. They instantiate this in an architecture called graph
isomorphism networks (GINs) where combination function are represented
as feedforward neural networks (FNNs) with ReLU activation. 
A more economic construction was proposed in \cite{WLTestReLULogarithmicResult}.
An important
caveat to these results is that they are \emph{non-uniform}, that is, the
construction of the GNN depends on the maximum size of the input
graphs.
A uniform result was obtained by \cite{bravoWLAnalyticalResult}, refining a result by \cite{OtherWLAnalyticalResult}. Their construction works with
any non-polynomial analytic activation function, such as logistic
sigmoid, and yields GNNs of internal dimension~1, that is, embeddings associate
vertices with real numbers rather than with vectors of reals. While,
remarkably, Bravo et al.\ show that in an $L$-layer GNN almost every
choice of real-valued weights suffices to match $L$-round 1-WL, they
do not establish that rational, or more generally finitely
representable, weights suffice.
It is interesting to contrast this with the result of \cite{KhalifePolynomialActivationResult} which states that it is impossible for a constant-depth GNN with a piecewise polynomial activation function such as ReLU
to match the expressive power of even 2-round 1-WL.

A complementary result was recently obtained by
\cite{GroheRepetition}. They show that recurrent GNNs with rational,
finite-precision weights and ReLU activation can attain the 1-WL
limit. That is, there is a recurrent GNN that distinguishes two vertices, or
graphs, if and only if some number of 1-WL rounds is able to
distinguish them. Note that recurrence significantly changes
the game: allowing the
number of GNN iterations to depend on the
input makes it possible to simulate message-passing algorithms, potentially
using a non-constant number of iterations for each round of the 1-WL algorithm.
When working with fixed-depth GNNs, in contrast, a single round of 1-WL must be simulated by only a constant number of GNN layers.

These results leave open a natural question, addressed in this
article: \emph{Can a constant-depth GNN uniformly match the
expressive power of 
%
%
1-WL while using only  rational
weights?}  We answer this question conditionally on Schanuel's
conjecture, a famous unsolved problem in transcendental number
theory. Assuming the conjecture, we show that for a broad class of
activation functions, including logistic sigmoid, hyperbolic tangent and
the exponential function $e^x$, $L$-layer GNNs with
rational weights can uniformly match the $L$-round 1-WL test, both at the vertex and at the
graph level. Like all results discussed above, we use sum aggregation.
Our construction yields  GNNs of internal dimension~1 and works even
with very simple weights: we may use self-weight~$1$,
aggregation-weight~$2$, and bias $0$ in every message-passing layer. This  also resolves, under Schanuel's conjecture, an open question from \cite{bravoWLAnalyticalResult} concerning the use of common weights in every layer.
More generally, our construction applies to a broad range of rational weight choices, just as the result of Bravo et al.\ applies to almost every choice of real-valued weights.
The appearance of Schanuel's conjecture can be understood from the
arithmetic of sum aggregation. If the scalar embeddings at a given
layer are linearly independent over $\mathbb{Q}$, then their sum
uniquely encodes any finite multiset. To iterate this argument, the activation function must produce embeddings that are again linearly independent over \(\mathbb{Q}\).
Consider the basic case of exponential activation $e^x$. If
$a_1,\ldots,a_n$ are embeddings and $m_1,\ldots,m_n$ are their
multiplicities in a multiset, then
\(
    e^{m_1a_1+\cdots+m_na_n}
    =
    (e^{a_1})^{m_1}\cdots(e^{a_n})^{m_n}.
\)
Thus, different multisets naturally give rise to different monomials
in the exponentials of the previous embeddings. Schanuel's conjecture
provides the transcendence properties needed to ensure that these new
values have the required independence.

Our connection to transcendental number theory also has a partial converse. In the slightly more expressive setting of graphs with
a certain kind of auxiliary vertices akin to virtual nodes \cite{gilmer2017neural,hu2020open}, we prove that the existence of an $L$-layer GNN with
exponentiation, sinh or cosh activation (and scalar embeddings on the final level) that attains the expressive power of $L$-round 1-WL implies the
restriction of Schanuel's conjecture to the numbers that arise in its
computation. This indicates that the connection between Schanuel's
conjecture and GNN expressivity goes in both directions rather than
being a mere artifact of our proof technique.

We complement our theoretical results with experiments. We evaluate GNNs of the form used in our construction, including simple configurations with integer weights, and observe that they indeed match the distinctions made by 1-WL. 

\paragraph{Related Work.}
There has been a lot of work on extending the expressive power of GNNs beyond 1-WL, using e.g.\ homomorphism counts of pattern graphs
\cite{BarceloPrecomputedHom, originalDHNpaper}, random node initializations \cite{randomNodeInit},
positional distance encodings \cite{GNNsWithDistanceEnc}
and mechanisms inspired by
higher-dimensional Weisfeiler-Leman tests \cite{provablyPowerfulGNNs, WLGoesNeural}.
Regarding \cite{bravoWLAnalyticalResult} and \cite{GroheRepetition}, we only add some minor clarifications over what was already said above. The results in 
\cite{bravoWLAnalyticalResult} 
concern graph distinguishability, but analogous results for the vertex 
level  can easily be derived from their proofs. The same is true for \cite{xu2018howPowerful}.
In contrast, \cite{GroheRepetition} have explicit results for the vertex level and for the graph level, on connected graphs. These are subject to the assumption that the
 graph size is provided as part of the initial embedding at every
vertex. Additional results on the graph level neither need this assumption nor connectedness, but require GNNs that have additional global aggregation in every iteration.
Our results also contrast those of \cite{khalifeIsUniformExpressiveness} which imply that 
GNNs with ReLU activation can express first-order properties that cannot be expressed by any GNN
activated by logistic sigmoid or hyperbolic tangent. Together with our results, assuming Schanuel's conjecture, this suggests that while logistic sigmoid can be more powerful than ReLU for exact distinguishability, ReLU can be more powerful than sigmoid for classification with a fixed margin.
Indeed, the results in \cite{khalifeIsUniformExpressiveness} require a uniform margin that separates  the embeddings on positive and negative instances, whereas our notion of distinguishability only requires these outputs to be distinct.
Schanuel's conjecture has also been used in previous work on machine
learning. Assuming the conjecture, training
\cite{trainingSigmoidNNs} and verification
\cite{DNNsVerificationSigmoid} have been shown to be decidable for
FNNs with sigmoid activation. Note that these results
concern computability rather than expressive power. We are not aware
of previous results connecting Schanuel's conjecture to the expressive
power of neural network architectures.
%
\section{Preliminaries}


\paragraph{Abstract Algebra.}

We use $\mathbb{N}$, $\mathbb{Z}$, $\mathbb{Q}$, $\bar{\mathbb{Q}}$, $\mathbb{R}$, $\mathbb{C}$
to denote natural, integer, rational, algebraic, real and complex numbers respectively. By $\mathbb{Q}_{\geq0}$ we reference non-negative rationals.
We recall that an algebraic number is a real number that is a root of a non-zero polynomial in one variable with rational coefficients.
For a vector $\bar v$, we use $\bar v[i]$ to denote its $i$-th component.
We also recall that a field consists of an underlying set $F$ and two binary operations $+, \cdot : F \times F \to F$, subject to additional conditions. As  customary, 
we denote a field by its underlying set, e.g.\ $\mathbb{Q}$ is used for the field of rational numbers
with the usual addition and multiplication operations.

For two fields $E$ and $F$, we say that $E$ is a \emph{field extension} of $F$, denoted $E/F$, if $F \subseteq E$
and the additive and multiplicative operations of $E$, when restricted to $F \times F$, are exactly
the additive and multiplicative operations of $F$. For instance, $\mathbb{R}$ is a field extension of $\mathbb{Q}$.
For a field extension $E/F$ and $S \subseteq E$,
we denote by $F(S)$ the 
field \emph{generated} by $F$ and $S$, that is, the 
smallest subfield $G$ of $E$
such that $S \subseteq G$ and $G/F$.  For instance, 
\(
    \mathbb{Q}(\sqrt{2})
    =
    \{a+b\sqrt{2} \mid a,b\in\mathbb{Q}\}
\)
is the smallest subfield of \(\mathbb{R}\) containing both
\(\mathbb{Q}\) and \(\sqrt{2}\).
Since $F(S)$ is a field,
$\frac{P(s)}{Q(s)} \in F(S)$ for all polynomials 
 $P(x)$ and $Q(x)$ over $F$ and all $s \in S$ with $Q(s) \neq 0$. 
This relates to a family of
generated fields used later in this article, namely, for any positive integer $\kappa$,
\[
    \mathbb{Q}(e^{1/\kappa})
    =
    \left\{
        \frac{P(e^{1/\kappa})}{Q(e^{1/\kappa})}
        \;\middle|\;
        P(x),Q(x) \text{ polynomials over } \mathbb{Q},\ Q(e^{1/\kappa})\neq 0
    \right\}.
\]
A field extension $E/F$ is \emph{finitely generated} if there is a finite $S \subseteq E$ such that $E = F(S)$. Clearly,  $\mathbb{Q}(\sqrt{2})$ and  $\mathbb{Q}(e^{1/\kappa})$ are finitely
generated.

Let $E/F$ be a field extension and $S = \{ s_1, \dots, s_n \} \subseteq E$.
Then $S$ is \emph{linearly independent} over $F$ if 
$a_1 s_1 + \dots + a_n s_n = 0$ in $E$ only if  $a_1 = \cdots = a_n = 0$, for all $a_1, \dots, a_n \in F$.  For instance, $\{1,\sqrt{2}\}$ is linearly independent over $\mathbb{Q}$. Similarly,
$S$ is \emph{algebraically independent} over $F$ if there is no non-zero
 polynomial $P(x_1, \dots, x_n)$ with coefficients in $F$ that satisfies $P(s_1, \dots, s_n) = 0$. For instance, $\{\sqrt{2}\}$ is not algebraically
independent over $\mathbb{Q}$ since $\sqrt{2}$ is a root of
$x^2-2$. In contrast, $\{\pi\}$ is algebraically independent over
$\mathbb{Q}$ since $\pi$ is transcendental.
A \textit{transcendence basis} of a field extension $E/F$ is a  subset of $E$, maximal w.r.t.\ set inclusion, that is 
algebraically independent over~$F$. For example, $\{\pi\}$ is a transcendence basis of
$\mathbb{Q}(\pi)/\mathbb{Q}$.
It is known that all transcendence bases share the same cardinality, see \cite{CladerRoss2025}. 
It is thus meaningful to define the \textit{transcendence degree} of $E/F$,
denoted $\mathsf{trdeg}_F(E)$, as the cardinality of any transcendence basis of $E/F$.
Thus
\(    \mathsf{trdeg}_{\mathbb{Q}}(\mathbb{Q}(\pi))=1
\)
and
\(
\mathsf{trdeg}_{\mathbb{Q}}(\mathbb{Q}(\sqrt{2}))=0.
\)
A field extension $E/F$ is \textit{algebraic} if
$\mathsf{trdeg}_F(E)=0$. Hence
$\mathbb{Q}(\sqrt{2})/\mathbb{Q}$ is algebraic.

The following consequence of Corollary~6.31 of
\cite{CladerRoss2025} will be useful.
\begin{lemma}\label{lemma_almost_trivial_alg_indep}
Let $E/F$ be a finitely generated field extension and let $S\subseteq E$.
If $E/F(S)$ is algebraic and
\(
    \mathsf{trdeg}_F(E) \geq |S|,
\)
then $S$ is algebraically independent over $F$.
\end{lemma}
\begin{proof}
Suppose that $S$ is algebraically dependent over $F$, and let
$T\subseteq S$ be maximal algebraically independent over $F$.
Then $|T|<|S|$, and $F(S)/F(T)$ is algebraic.
Since $E/F(S)$ is algebraic, so is $E/F(T)$.
Point~1 of Corollary~6.31 of \cite{CladerRoss2025} states that
if $E/F$ is finitely generated and $E/F(T)$ is algebraic, then
\(
    \mathsf{trdeg}_F(E) \leq |T|.
\)
Consequently,
\(
    \mathsf{trdeg}_F(E) \leq |T| < |S|,
\)
contradicting $\mathsf{trdeg}_F(E)\geq |S|$.
\end{proof}
A classical result from transcendental number theory provides an important way of obtaining algebraically independent numbers, from \cite{Weierstrass1885Lindemann}.
\begin{theorem} \textup{(Lindemann–Weierstrass)}\label{theorem_LW}
    If $\{z_1, \dots, z_n \} \subseteq \bar{\mathbb{Q}}$ is linearly independent over $\mathbb{Q}$, then
    $\{e^{z_1}, \dots, e^{z_n} \}$ is algebraically independent over $\mathbb{Q}$.
\end{theorem}

The following conjecture, due to Schanuel and first published by \cite{Lang1966TranscendentalNumbers}, generalizes the Lindemann--Weierstrass theorem by allowing the exponents to be arbitrary complex  numbers, rather than only algebraic ones.
\begin{conjecture} \textup{(Schanuel's conjecture)} 
    If $\{z_1, \dots, z_n \} \subseteq \mathbb{C}$ is linearly independent over $\mathbb{Q}$, then   $\mathsf{trdeg}_{\mathbb{Q}}(\mathbb{Q}(z_1, \dots, z_n, e^{z_1}, \dots, e^{z_n})) \geq n$.
\end{conjecture}
Note that Theorem~\ref{theorem_LW} is the special case of Schanuel's conjecture where $z_1,\dots,z_n$ are algebraic. By definition, algebraic numbers contribute no transcendence degree 
and thus $\mathsf{trdeg}_{\mathbb{Q}}(\mathbb{Q}(z_1, \dots, z_n, e^{z_1}, \dots, e^{z_n})) \geq n$
implies $\mathsf{trdeg}_{\mathbb{Q}}(\mathbb{Q}(e^{z_1}, \dots, e^{z_n})) \geq n$,
which is the same as saying that $\{e^{z_1}, \dots, e^{z_n} \}$ is algebraically independent over $\mathbb{Q}$.

Schanuel’s conjecture naturally relates to activation functions defined in terms of exponentials such as sigmoid. We identify a concrete class of such functions that we shall analyze in this article. 
We say that a function $\sigma:\mathbb{R}\to\mathbb{R}$  is \textit{Schanuel-ready}
if it is of the form $\sigma(x) = \frac{P(e^x)}{Q(e^x)}$, with $P(y)$ and $Q(y)$ polynomials with rational coefficients, such that
 (i)~$\frac{P(y)}{Q(y)}$ is non-constant, positive and strictly increasing on $y > 1$ and 
    (ii)~$Q(y) \neq 0$ for all $y > 0$.
Notable Schanuel-ready functions include the exponentiation function $\exp = e^x$,
logistic sigmoid $\frac{e^x}{e^x+1}$, hyperbolic tangent $\frac{e^{2x} - 1}{e^{2x} + 1}$,
hyperbolic sine $\frac{e^{2x} - 1}{2e^x}$, and hyperbolic cosine $\frac{e^{2x} + 1}{2e^{x}}$. 

\paragraph{Graphs.}
We work with directed graphs $G=(V,E)$ throughout this article. For each $v \in V$ we define the \emph{neighborhood} $N(v)=\{ u \in V \mid (v, u) \in E \}$
to be the set of vertices reachable from $v$ in one step. A \emph{pointed graph} is a pair $(G,v)$
with  $G=(V,E)$ a graph and $v \in V$ a distinguished vertex. An
\emph{embedded graph of dimension} $d >0$ is
a triple $(V,E,\lambda)$ with $\lambda: V \rightarrow \mathbb{R}^d$
an \emph{embedding}. Pointed embedded graphs are defined as expected, taking the form $((V,E,\lambda),v)$.


\paragraph{AC-GNNs and 1-WL.}
An \emph{$L$-layer aggregate-combine graph neural network (AC-GNN)} is
a tuple 
$\gnn=(\sigma,(C_i)_{1 \leq i \leq L},(A_i)_{1 \leq i \leq L},(b_i)_{1 \leq i \leq L},I)$ where  $\sigma:\mathbb{R}\rightarrow\mathbb{R}$ is an \emph{activation function} applied component-wise, there are dimensions $d_0,\dots,d_L \geq 1$ such that
$A_i, C_i \in \mathbb{Q}^{d_{i} \times d_{i-1}}$ and $b_i \in \mathbb{Q}^{d_i}$ for $1 \leq i \leq L$,
and $I$ specifies an \emph{initialization function}. More precisely, $I$ is a tuple
$(B_1,\dots,B_k,c_1,\dots,c_k)$
such that
there are $d'_0,\dots,d'_k \geq 1$ with $B_i \in \mathbb{Q}^{d'_{i} \times d'_{i-1}}$ and $c_i \in \mathbb{Q}^{d'_i}$ for $1 \leq i \leq k$.  We require that
$d'_k=d_0$ and call $d'_0$ the \emph{input dimension} of $\gnn$, and $d_L$ its \emph{output dimension}. The \emph{weights} of a GNN are the elements of the $A_i, B_i, C_i, b_i$ and $c_i$.

As the input to GNNs, we 
use graphs endowed with non-negative rational vertex features. Formally,
a \emph{featured graph of dimension} $n$ is an embedded graph $(V,E,\lambda)$ of dimension $n$ such that the image of $\lambda$ contains only vectors of non-negative rationals.
Let $\gnn$ be a GNN as above, and let $G=(V,E,\lambda)$ be a featured graph of dimension $d'_0$, the input dimension of $\gnn$. When applying $\gnn$ to $G$, in the first step $I$ transforms $G$ into
an embedded graph $(V,E,\lambda^{(0)})$ of dimension $d_0$.
Informally,  $I$ should be thought of as a feedforward neural network. Formally, we set
$\widetilde{\lambda}^{(0)}(\bar x) = \sigma(B_k\sigma( \dots \sigma(B_1 \bar x +c_1) \dots) + c_k)$ for all $\bar x \in \mathbb{R}^{d'_0}$ and
$\lambda^{(0)}(v)=\widetilde{\lambda}^{(0)}(\lambda(v))$ for all $v \in V$. Next, $\gnn$ constructs a sequence of $L$ embeddings of $(V,E)$, starting from $\lambda^{(0)}$.
Let $1 \leq i \leq L$.
Set 
$\widetilde{\lambda}^{(i)}(\bar x, M) = \sigma(C_i\bar x + A_i\sum M + b_i)$ 
for all $\bar x \in \mathbb{R}^{d_{i-1}}$
and all finite multisets $M$
over $\mathbb{R}^{d_{i-1}}$. Then define $
\lambda^{(i)}(v)=\widetilde{\lambda}^{(i)}(\lambda^{(i-1)}(v),\msl\lambda^{(i-1)}(u) \mid u \in N(v)\msr)$ for all $v \in V$. For vertex classification, one may now work with the
embedded graph $(V,E,\lambda^{(L)})$.
For graph classification, a final \emph{graph-level readout} is carried out to consolidate the graph embedding into a single vector  $\lambda_G = \sum_{v \in V} \lambda^{(L)}(v)$. We refer to the functions $\widetilde{\lambda}^{(0)}, \dots, \widetilde{\lambda}^{(L)}$ as \emph{update functions}.

We remark that the initialization function is a slightly unusual feature of our GNN model. Its purpose is to synchronize the layers of the GNN with the rounds of the Weisfeiler--Leman algorithm. In principle, the initialization function could instead be implemented by additional GNN layers, at the cost of increasing the depth of the network.

We next introduce the 1-dimensional Weisfeiler--Leman algorithm (1-WL), in its color refinement variant; see also \cite{GroheLogic}.
The algorithm takes a featured graph $G=(V,E,\lambda)$ as input and proceeds in rounds, producing increasingly refined  colorings of the vertices of $G$. The initial coloring is 
defined by $\mu^{(0)}(v) = \mathsf{hash}(\lambda(v))$ where $ \mathsf{hash}$ is an injective function whose range is a countable set of \emph{colors}. In later rounds $i>0$, the coloring is defined by 
\begin{equation*}
    \mu^{(i)}(v) = \mathsf{hash}\Big(\mu^{(i-1)}(v), \msl \mu^{(i-1)}(u) \mid u \in N(v) \msr \Big).
\end{equation*}
When used as an approximative test for graph isomorphism, 1-WL is typically run until stabilization, that is, 
until $\mu^{(i)}(v_1) = \mu^{(i)}(v_2)$
implies $\mu^{(i+1)}(v_1) = \mu^{(i+1)}(v_2)$
for all vertices $v_1,v_2$. Here, we will mostly be interested in the coloring $\mu^{(L)}(v)$ for some fixed $L$. This coloring can be lifted to the graph level by taking the color of $G=(V,E)$ to be $\mu_G=\mathsf{hash}(\msl 
\mu^L(v) \mid v \in V
\msr)$.

Both the $L$-round 1-WL algorithm and every $L$-layer AC-GNN induce an equivalence relation on all graphs: $G_1$ and $G_2$ are equivalent if
$\mu(G_1)=\mu(G_2)$
or
$\lambda(G_1)=\lambda(G_2)$, respectively. Equivalence relations on pointed graphs $(G,v)$ are induced in the same way, via 
 $\mu^{(L)}(v)$ and
$\lambda^{(L)}(v)$.
It follows from \cite{xu2018howPowerful, WLGoesNeural} that $L$-layer GNNs cannot classify graphs or vertices
as being non-equivalent unless the $L$-round 1-WL  does  the same.
The following, which can be inferred from Theorem~3 of \cite{xu2018howPowerful}, states a converse. 
If the update functions of a GNN $\gnn$ are injective, then the equivalence relation on pointed graphs induced by $\gnn$ is exactly that of 1-WL.
The same holds on the graph level if we additional demand that the image of the final update function is linearly independent over $\mathbb{Q}$.

\begin{lemma} \label{lemma_prelim_lin_indep_implies_wl}
    Let $\gnn$ be an $L$-layer AC-GNN such that 
$\widetilde{\lambda}^{(0)},
\dots, \widetilde{\lambda}^{(L)}$ are  injective.
    Then for all pointed graphs $(G_1,v_1)$ and $(G_2,v_2)$, 
    $\lambda^{(L)}(v_1)=
    \lambda^{(L)}(v_2)$
    if and only if 
    $\mu^{(L)}(v_1)=
    \mu^{(L)}(v_2)$.
    If, in addition, every finite subset of the image of
$\widetilde{\lambda}^{(L)}$ is linearly independent over $\mathbb{Q}$,
    then for all graphs
    $G_1,G_2$, $\lambda_{G_1}=
    \lambda_{G_2}$
    if and only if 
    $\mu_{G_1}=
    \mu_{G_2}$.
\end{lemma}
For our later use of Lemma~\ref{lemma_prelim_lin_indep_implies_wl}, we remark
that it suffices for the lemma to hold that the functions $\widetilde{\lambda}^{(0)},
\dots, \widetilde{\lambda}^{(L)}$ are only \emph{effectively injective}. For  $1 \leq i \leq L$, this means that $\widetilde{\lambda}^{(i)}$ is injective on the restriction to inputs $(\bar x,M)$ such that there is a graph $G=(V,E)$ and a vertex $v \in V$ with $\bar x=\lambda^{(i-1)}(v)$
and $M=\msl\lambda^{(i-1)}(u) \mid u \in N(v)\msr$. Likewise, $\widetilde{\lambda}^{(0)}$ is effectively injective if it is injective on vertex labels, that is, on vectors of non-negative rationals.

\section{Matching the upper bound}\label{section_main_result}

The following theorem states the main result of this article. When speaking of a GNN of \emph{internal dimension~1}, we mean that all dimensions $d_0,\dots,d_L$ from the definition of GNNs are~1 (while the dimensions $d'_1,\dots,d'_k$ of the initialization function need not be~1).
\begin{theorem}
\label{thm:main}
  Assume that Schanuel's conjecture is true. Then 
  for every Schanuel-ready activation function $\sigma$ and every $L \geq 1$, there is an $L$-layer GNN $\gnn$ 
  of internal dimension 1 that uses $\sigma$ activation and integer weights such that
  %
 (i)~for all pointed graphs,  $(G_1,v_1)$ and $(G_2,v_2)$, 
    $\lambda^{(L)}(v_1)=
    \lambda^{(L)}(v_2)$
    if and only if 
    $\mu^{(L)}(v_1)=
    \mu^{(L)}(v_2)$;
    and (ii)~for all graphs   $G_1,G_2$, $\lambda_{G_1}=
    \lambda_{G_2}$
    if and only if 
    $\mu_{G_1}=
    \mu_{G_2}$.
\end{theorem}
In fact, our result is even stronger than what is stated in Theorem~\ref{thm:main}. Note that, for GNNs with internal dimension~1, the matrices $A_i,C_i$ and bias vectors $b_i$ are simply scalars. Our proof shows that it suffices to choose $C_i=1$, $A_i=2$, and $b_i=0$ for every message-passing layer. For featured input graphs of dimension $N$, the matrices $B_i$ in the initialization function use only the values $2^1,2^2,\dots,2^N$, and the vectors $c_i$ are either comprised entirely of zero or of one. 
Thus the weights in $\gnn$ are not only integers, but also reside in a finite set of natural numbers whose size only depends on $N$, but not on $L$. In the general setting, our construction allows all positive rational $A_i$ and $C_i$ satisfying$\frac{C_i}{A_i} \not\in \mathbb{N}$. In this sense, our result applies to a broad range of rational weights.

We now give the technical details of our construction. Say that an $L$-layer GNN $\gnn=(\sigma,(C_i)_{1 \leq i \leq L},(A_i)_{1 \leq i \leq L},(b_i)_{1 \leq i \leq L},I)$ of input dimension $N$ and internal dimension~1 is \textit{Schanuel-ready} if it uses a Schanuel-ready activation function and satisfies the following conditions:
\begin{enumerate}
    \item the initialization function is
   \(I=(B_1,B_2,c_1,c_2)\), where
   $B_1$ is the $N \times N$ identity matrix, $c_1=\mathbf{1}_N$, 
   $B_2=(2^1,2^2,\ldots,2^N)$, and $c_2=0$. Then 
  $\lambda^{(0)}(v) = 
\sigma\big(
\sum_{j=1}^{N} 2^j \sigma(\lambda(v)[j] + 1)
\big)$ 
for all graphs $G=(V,E)$ and vertices $v \in V$;
    \item  $A_i, C_i > 0$ and $b_i \geq 0$ are rational numbers such that $\frac{C_i}{A_i} \not\in \mathbb{N}$. 
\end{enumerate}
The purpose of the above initialization function $I$
is to encode a vector of non-negative rational numbers, as given by the input features, into a single real number. One can show that, under Schanuel's conjecture, the values $\sigma(q+1)$ are linearly independent over
$\mathbb{Q}$, preventing different feature values from
canceling out in the sum. The powers $2^j$ provide, for each rational value, a binary encoding of the positions in which it occurs. 
In the appendix, we prove the following.
\begin{restatable}{proposition}{mainSchanuelsProposition}\label{main_schanuels_proposition}
    Assuming Schanuel's conjecture is true,
    every $L$-layered Schanuel-ready GNN $\gnn$
    satisfies  $\lambda^{(L)}(v_1)=
    \lambda^{(L)}(v_2)$
    if and only if 
    $\mu^{(L)}(v_1)=
    \mu^{(L)}(v_2)$, for all pointed graphs  $(G_1,v_1)$ and $(G_2,v_2)$.
\end{restatable}
In what follows, we present
a proof of a restricted version of Proposition~\ref{main_schanuels_proposition} that is less technical, yet still reveals the underlying mechanisms.
More concretely, we assume that
  the activation function is $\sigma = \exp$, $C_i=1$, $A_i=2$, and $b_i=0$ for every layer $i$, and all feature vectors in input graphs are one-hot, that is, they have exactly one non-zero component.
The  assumption of one-hot feature vectors also allows us to simplify the initialisation to be $I=(B_1,c_1)$, where
$B_1=(1,2,\ldots,N)$ and $c_1=0$. Then
\(
    \lambda^{(0)}(v)
    =
    \sigma\big(
        \sum_{j=1}^n j\,\lambda(v)[j]
    \big)
\)
for all graphs $G=(V,E)$ and vertices $v\in V$.


We now define, for every layer~$i$ of our GNN $\gnn$, a set of vectors $Z_i$ that overapproximates the set of embedding vectors that can arise in layer~$i$ of $\gnn$, across all graphs. Since the sets $Z_i$ can be infinite and Schanuel's conjecture applies only to finite sets, we shall also define, for each $r \geq 1$, a finite set $Z_{i,r}$. The reader should think of the set $Z_{1,2r+1}$ in the same way as $Z_i$, but on graphs of out-degree at most $r$. We proceed inductively:
$$
\begin{array}{r@{\;}c@{\;}lcr@{\;}c@{\;}l}
 \widehat{Z}_{0, r} &=& \{ e \} && Z_{0, r} &=& \{ e^1, \dots, e^r \},\\[1mm]
 \widehat{Z}_{i+1, r} &=& \{ e^z \mid z \in Z_{i, r} \} && Z_{i+1, r} &=& \{z_1 \cdots z_r \mid z_1, \dots, z_r \in \widehat{Z}_{i+1, r} \cup \{ 1 \}
    \text{ not all } 1\}.
\end{array}
$$
Importantly, $z_1,\dots,z_r$ do not need to be pairwise distinct.
Set $Z_{i} = \bigcup_{r \geq 1} Z_{i, r}$ for $0 \leq i \leq L$.

\begin{lemma}
    For all featured graphs $G$, all vertices $v \in V$ and  $0 \leq i \leq L$ we have $\lambda^{(i)}(v) \in Z_{i}$.
\end{lemma}
\begin{proof}
    We proceed by induction on $i$. The base case is immediate as the image of the initialisation function $\lambda^{(0)}$ is $\{e^1, \dots, e^N \} \subseteq Z_{0}$.
    For the inductive step assume that $\lambda^{(i)}(v) \in Z_{i}$ for all $v \in V$.
    Then, aggregation and combination in $\gnn$ amount to
    $\sigma\big(z_0 + 2 (z_1 + \dots + z_n) \big)$ with $z_0, \dots, z_n \in Z_{i}$.
    Applying $\sigma$ yields $(e^{z_0}) \cdot (e^{z_1})^2 \cdot {\dots} \cdot (e^{z_n})^2$,
    which is in $Z_{i+1}$ as $e^{z_0}, \dots, e^{z_n} \in \widehat{Z}_{i+1}$.
\end{proof}

All that is left to show at this point is that 
each finite subset of $Z_i$ is linearly independent over $\mathbb{Q}$ for all $0 \leq i \leq L$.
Indeed, then one need only apply Lemma~\ref{lemma_prelim_lin_indep_implies_wl} noting that
$\widetilde{\lambda}^{(0)}(s) = \sigma((1, \dots, N) \times z)$ and
$\widetilde{\lambda}^{(i)}(s, M) = \sigma(s + \sum2M)$ are injective when $z$ is one-hot, $s \in Z_{i-1}$,
and $M$ is a finite multiset over $Z_{i-1}$
(see appendix for details).
To prove the required linear  independence of finite subsets of $Z_i$, we first observe that it follows from algebraic independence
of finite subsets of $\widehat{Z}_{i} = \bigcup_{r \geq 1} \widehat{Z}_{i, r}$ over $\mathbb{Q}$. For technical reasons, we establish a stronger result:
\begin{lemma}\label{claim_alg_indep_to_lin}
    Let $0 \leq i \leq L$. If each finite subset of $\bigcup_{j=0}^i\widehat{Z}_{j}$ is algebraically independent over $\mathbb{Q}$,
    then each finite subset of $\bigcup_{j=0}^i Z_{j} \cup \{ 1\}$ is linearly independent over $\mathbb{Q}$.
\end{lemma}
\begin{proof}
    Suppose for contradiction that some finite  $S \subseteq \bigcup_{j=0}^i Z_{j} \cup  \{ 1\} $
    is linearly dependent over $\mathbb{Q}$.
    Thus, for $S \setminus \{ 1 \} = \{z_1  \dots, z_n \}$ there are 
    $c_1, \dots, c_n, d \in \mathbb{Q}$ not all zero
    such that  $\sum_{j=1}^n c_j z_j + d = 0$.
    Each value $z_j$ is a finite product of elements in
$\bigcup_{\ell=0}^i \widehat Z_\ell$. This follows directly from the
definition for $z_j\in Z_\ell$ with $\ell\geq 1$, while every element
$e^k\in Z_0$ is the product of $k$ copies of the unique element
$e\in\widehat Z_0$. We can thus write
    $z_j = s_1^{b_{1, j}} \cdot {\dots} \cdot s_m^{b_{m, j}}$ where $b_{1, j}, \dots, b_{m, j} \in \mathbb{N}$ not all zero
    and $s_1, \dots, s_m \in \bigcup_{j=0}^i \widehat{Z}_{j}$.
    But then, $\sum_{j=1}^n c_j s_1^{b_{1, j}} \cdot {\dots} \cdot s_m^{b_{m, j}} + d = 0$.
    Hence, by algebraic independence of $s_1, \dots, s_m$ over $\mathbb{Q}$,
    the function $\sum_{j=1}^n c_j x_1^{b_{1, j}} \cdot {\dots} \cdot x_m^{b_{m, j}} + d$ must be the zero polynomial.
    But since $z_1, \dots, z_n$ are distinct, each vector of exponents
    $(b_{1, j}, \dots, b_{m, j})$ must also be distinct.
    In fact, $(b_{1, j}, \dots, b_{m, j}) \neq (0, \dots, 0)$.
    Thus, no cancellations take place in the expression above, resulting
    in a non-zero polynomial which
    contradicts our initial assumption.
\end{proof}
At this point, it remains to show the following.
\begin{lemma}
    Assuming that Schanuel's conjecture is true,
    all finite subsets of $\bigcup_{j=0}^L \widehat{Z}_{j}$ are algebraically independent over $\mathbb{Q}$.
\end{lemma}
\begin{proof}
    Assuming Schanuel's conjecture, we prove by induction on $i$ that, for $0 \leq i \leq L$ and $r \geq 1$,
    \begin{enumerate}
        \item $\bigcup_{j=0}^{i} \widehat{Z}_{j, r}$ is algebraically independent over $\mathbb{Q}$ and
        \item the sets $\widehat{Z}_{0, r}, \dots, \widehat{Z}_{i, r}$ are pairwise disjoint; and so are $\{1\}, Z_{0, r}, \dots, Z_{i, r}$.
    \end{enumerate}
    This clearly implies the intended result.
    For the base case note that $\bigcup_{j=0}^0 \widehat{Z}_{j} = \{e\}$.
    Thus Condition~1 holds trivially as $e$ is a transcendental number.
    Condition 2 follows from the above and the fact that $Z_{0, r}$ are positive powers of the transcendental number $e$,
    thus ensuring $1 \not\in Z_{0, r}$.
    
    Now suppose Conditions~1 and~2 hold for $i \geq 0$.
    We show that the same is true for $i+1$.
    By Condition~1 of our induction hypothesis, $\bigcup_{j=0}^i \widehat{Z}_{j, r}$ is algebraically independent over $\mathbb{Q}$.
    Thus by Lemma~\ref{claim_alg_indep_to_lin}, the set
    $S = \bigcup_{j=0}^{i} Z_{j, r} \cup \{ 1 \}$ is linearly independent over $\mathbb{Q}$.
    Applying Schanuel's conjecture to $S$ yields
    $F=\mathbb{Q}(S \cup \bigcup_{j=0}^{i+1} \widehat{Z}_{j, r})$ such that
    $\mathsf{trdeg}_{\mathbb{Q}}(F) \geq |S|$.
    We next analyse the size of $S$.
    Notice, first, that $|Z_{j, r}| = |\widehat{Z}_{j+1, r}|$ for all $j$.
    Moreover, by Condition~2 of our induction hypothesis, $\{1\}, Z_{0, r}, \dots, Z_{i, r}$
    are all pairwise disjoint. Thus,
    \begin{equation*}
        |S| = 
        |\bigcup_{j=0}^{i} Z_{j, r} \cup \{ 1 \}| = 
        \sum_{j=0}^i |Z_{j, r}| +1 =
        \sum_{j=1}^{i+1} |\widehat{Z}_{j, r}| + |\widehat{Z}_{0, r}| = \sum_{j=0}^{i+1} |\widehat{Z}_{j, r}|.
    \end{equation*}
    Write $E = \mathbb{Q}(\bigcup_{j=0}^{i+1} \widehat{Z}_{j, r})$.
    Notice that $S \subseteq E$ because elements of $S$ are monomials over $\bigcup_{j=0}^{i} \widehat{Z}_{j, r}$ or $1$.
    Hence, $E = F$ and
    $\mathsf{trdeg}_{\mathbb{Q}}(E) \geq |S| = \sum_{j=0}^{i+1} |\widehat{Z}_{j, r}|$.
    With Lemma~\ref{lemma_almost_trivial_alg_indep}, 
    we obtain that $\bigcup_{j=0}^{i+1} \widehat{Z}_{j, r}$ is algebraically independent over $\mathbb{Q}$,
    thus establishing Condition 1.

   The $\widehat{Z}_{0, r}, \dots, \widehat{Z}_{i+1,
      r}$
    part of  Condition~2 follows from the facts that the sets $\{1\}, Z_{0, r}, \dots, Z_{i, r}$
    are pairwise disjoint, due to the induction hypothesis, and exponentiation is injective.
    Now, suppose that $Z_{i+1, r} \cap Z_{j, r} \neq \emptyset$ for $0
    \leq j \leq i$, 
    with $z$ a common element. Clearly, $z$ is then a monomial over
    $\widehat{Z}_{i+1, r}$ as well as over $\widehat{Z}_{j, r}$.
    But by our argument above, $\widehat{Z}_{i+1, r} \cap \widehat{Z}_{j, r} = \emptyset$.
    Additionally, by Condition~1, already established above for $i{+}1$, $\widehat{Z}_{i+1, r} \cup \widehat{Z}_{j, r}$ is algebraically independent over~$\mathbb{Q}$.
    Hence, no monomial over $\widehat{Z}_{i+1, r}$ equates to a monomial over $\widehat{Z}_{j, r}$, a contradiction. 
    Finally, $Z_{i+1, r} \cap \{1\} = \emptyset$ follows from the fact that
    $1$ is algebraic whilst the elements of $Z_{i+1, r}$ are not.
\end{proof}

\section{The Transcendence of GNN Computations}\label{section_converse_direction}

The aim of this section is to provide a partial converse to Theorem~\ref{thm:main}.
In doing so we argue that the connection between GNN expressivity
and Schanuel's conjecture goes both ways, and is not merely an artefact of our technique.
To achieve this, we work in a slightly more expressive setting that admits virtual
nodes. These are auxiliary vertices commonly used in GNNs to globally aggregate information, see for instance \cite{gilmer2017neural,hu2020open}. We use a variant tailored to our purposes  in which virtual nodes need not be globally connected and, crucially, 
use a modified update function that does not depend on their own state.
We also make the additional assumption that the embedding computed by the final GNN layer is of dimension~1.

More precisely, a \emph{featured graph with virtual nodes} is a triple
$G=(V,E,\lambda)$, where
$
\lambda:V\to\mathbb{Q}_{\geq 0}^{N}\cup\{\mathsf{virtual}\}.
$
Vertices $v$ with $\lambda(v)=\mathsf{virtual}$ are called \emph{virtual},  all other vertices are  \emph{ordinary}.
We extend our definition of GNNs accordingly. An \emph{$L$-layer GNN with virtual nodes} is a tuple
$$
\gnn=
\bigl(
\sigma,(C_i)_{1\leq i\leq L},(A_i)_{1\leq i\leq L},
(b_i)_{1\leq i\leq L},
(A'_i)_{1\leq i\leq L},(b'_i)_{1\leq i\leq L},
I
\bigr),
$$
where $\sigma,C_i,A_i,b_i$ and $I$ are as in the definition of  GNNs with dimensions $d_0,\ldots,d_L$,  
$
A'_i\in\mathbb{Q}^{d_i\times d_{i-1}}$
and
$b'_i\in\mathbb{Q}^{d_i}
$ for $1\leq i\leq L$.
For every node $v$ with
$\lambda(v)=\mathsf{virtual}$, the initial embedding $\lambda^{(0)}(v)$
is some vector shared by all virtual nodes, whilst
$
\lambda^{(i)}(v)
=
\sigma\big(
A'_i\sum_{u\in N(v)}\lambda^{(i-1)}(u)+b'_i
\big).
$

We also consider a \emph{1-WL test over graphs with virtual nodes}: for vertices 
$u$ with $\lambda(u) = \mathsf{virtual}$, we define
$\mu^{(i+1)}(u) = \textsc{hash}(\mathsf{virtual}, \msl \mu^{(i)}(v) \mid v \in N(u) \msr)$ where $\mathsf{virtual}$
is a distinguished value that  is never assigned to an ordinary vertex.
The following is routine to prove, see the appendix.
\begin{restatable}{proposition}{propWLovervirtual}\label{prop_1_WL_over_virtual}
    Suppose $(G, v)$ and $(H, u)$ are two featured pointed graphs with virtual nodes.
    Then, for all $L$-layered GNNs $\gnn$,
     $\mu^{(L)}(v) = \mu^{(L)}(u)$ implies
    $\lambda^{(L)}(v) = \lambda^{(L)}(u)$ and
    $\mu_G = \mu_H$ implies $\lambda_G = \lambda_H$.
\end{restatable}

We are now in a position to establish the advertised converse.
\begin{restatable}{theorem}{propConverse}\label{prop:converse}
    Let $\gnn$  be an $L$-layer GNN that uses  $\sigma \in \{ \exp, \sinh, \cosh \}$
    as activation
    and produces 1-dimensional vertex embeddings in its final layer.
    Further assume that $\lambda_{G_1} = \lambda_{G_2}$ if and only if $\mu_{G_1} = \mu_{G_2}$
    for all featured graphs with virtual nodes $G_1,G_2$.
    Then for every featured graph $G=(V,E,\lambda)$ the values $\big\{ \exp\big(A_L'\lambda^{(L-1)}(v) \big) \mid v \in V  \big\}$
    are algebraically independent over $\mathbb{Q}$. 
\end{restatable}
\begin{proof}
     We only treat $\sigma = \exp$ here.
    The cases  $\sigma \in \{ \sinh, \cosh \}$
    are 
    relegated to the appendix.
    Observe that, for every  featured graph with virtual nodes $G=(V,E,\lambda)$ and every $u \in V$ with $\lambda(u) = \mathsf{virtual}$, 
$
\lambda^{(L)}(u)
=
\exp\big(
A'_L\sum_{v\in N(u)}\lambda^{(L-1)}(v)+b'_L
\big)
=
\exp(b'_L)
\prod_{v\in N(u)}
\exp\left(A'_L\lambda^{(L-1)}(v)\right).
$
Thus, the final embedding of a virtual node is $\exp(b'_L)$ times
a monomial in the values
$
\mathcal{V}=\big\{ \exp\big(A_L'\lambda^{(L-1)}(v) \big) \mid v \in V \big\}$ where the exponent of each value records how often it occurs among
the neighbors of the virtual node.
Towards a contradiction, assume that for some graph $G=(V,E,\lambda)$ this set is not
algebraically independent over $\mathbb{Q}$.
We construct two graphs $G_Q$ and $G_R$ that are distinguishable
by $L$-round 1-WL but not by $\gnn$, contradicting our assumption.  
        
    Let $z_1, \dots, z_n$ be an enumeration of $\mathcal{V}$.
    For $1 \leq i \leq n$ choose a vertex $v_i \in V$
    such that $z_i =
    \exp(A_L'\lambda^{(L-1)}(v_i))$. 
      Since the $z_i$ are pairwise distinct, so are the embeddings
    $\lambda^{(L-1)}(v_i)$.
   %
By~Proposition~\ref{prop_1_WL_over_virtual}, the pointed graphs $(G, v_1), \dots, (G, v_n)$
    are all pairwise distinguishable by the $(L{-}1)$-round 1-WL algorithm.
    Now, since $\{z_1, \dots, z_n\}$ is algebraically dependent over $\mathbb{Q}$,
    there is a non-zero polynomial $P(x_1, \dots, x_n)$
    such that $P(z_1, \dots, z_n) = 0$. Multiplying by a common denominator, we
    may assume that $P$ has integer coefficients. 
    %
    %
    Moving all terms with negative coefficients to the right hand side, we obtain a polynomial equality $Q(x_1, \dots, x_n) = R(x_1, \dots, x_n)$ where $Q$ and $R$ have disjoint sets of monomials and non-negative integer coefficients. 
    Exchanging $Q$ and $R$ if necessary, we may assume that $R$
    has no constant term.
%
    Multiplying by $\exp(b'_L)$ and writing
    out the monomials gives
    \begin{equation*}
        \sum_{i=1}^{q}
        \alpha_i
        z_1^{\beta_{i,1}}\cdots z_n^{\beta_{i,n}}
        \exp(b'_L)
        +
        \alpha_{q+1}\exp(b'_L)
        =
        \sum_{j=1}^{r}
        \gamma_j
        z_1^{\delta_{j,1}}\cdots z_n^{\delta_{j,n}}
        \exp(b'_L),
    \end{equation*}
    where $\alpha_i,\gamma_j\in\mathbb{N}\setminus\{0\}$, with the
    exception that $\alpha_{q+1}$ may be zero, and the exponent
    vectors
    \(
        (\beta_{i,1},\ldots,\beta_{i,n})\)
        and
        \( (\delta_{j,1},\ldots,\delta_{j,n})
    \)
    are not all zero entries.
    The desired graphs $G_Q$ and $G_R$  realize the two
    sides of this equality. A monomial
    $z_1^{\beta_{i,1}}\cdots z_n^{\beta_{i,n}}$ is realized by a
    virtual node having, for every $k$, $\beta_{i,k}$ neighbors in
    fresh copies of $(G,v_k)$. By the initial observation in this
    proof, its final embedding is
    \(
        \exp(b'_L)
        z_1^{\beta_{i,1}}\cdots z_n^{\beta_{i,n}}.
    \)
    An isolated virtual node, in turn, has final embedding
    $\exp(b'_L)$ and thus realizes the constant monomial.
    To ensure that all vertices belonging to the copies of $G$
    make the same contribution on both sides, we additionally put
    on each side copies corresponding to the monomials realized on
    the other side. Formally, $G_Q$ is defined as follows:
    \begin{enumerate}
        \item for $1\leq i\leq q$, add $\alpha_i$ disjoint copies
        of the graph constructed as follows:
        \begin{enumerate}
            \item for $1\leq k\leq n$, add $\beta_{i,k}$ fresh
            copies of the pointed graph $(G,v_k)$;
            \item create a fresh virtual node $u$ and add an edge
            from $u$ to the distinguished vertex in each copy
            added in the previous step;
        \end{enumerate}
        \item add $\alpha_{q+1}$ fresh isolated virtual vertices;
        \item for $1\leq j\leq r$ and $1\leq k\leq n$, add
        $\gamma_j\delta_{j,k}$ disjoint copies of $(G,v_k)$.
    \end{enumerate}
    The graph $G_R$ is defined symmetrically: interchange
    $q$ and $r$ as well as $(\alpha,\beta)$ and $(\gamma,\delta)$,
    and omit Step~2. Notice that the newly added edges point from
    virtual nodes into the copies of $G$ and thus do not change the
    neighborhoods, and hence the GNN embeddings or 1-WL colors, of
    vertices inside these copies. We argue in the appendix that $G_Q$ and $G_R$ have the desired properties.
%
\end{proof}
Theorem~\ref{prop:converse}
immediately yields the following restriction of
 Schanuel's conjecture
 to GNN outputs.
\\ 

\begin{corollary}
\label{cor:restricted-schanuel}
Let $\gnn$ be as in Theorem~\ref{prop:converse}. Then,
for every featured graph $G=(V,E,\lambda)$ and all distinct
\(
    x_1,\ldots,x_n
    \in
    \{A'_L\lambda^{(L-1)}(v)\mid v\in V\},
\)
we have
\(
    \mathsf{trdeg}_{\mathbb Q}
    \mathbb Q(x_1,\ldots,x_n,e^{x_1},\ldots,e^{x_n})
    \geq n.
\)
Consequently, Schanuel's conjecture holds when restricted to
$\mathbb Q$-linearly independent tuples arising in this way.
\end{corollary}
\begin{proof}
Let $x_1,\ldots,x_n$ be as in the statement. By Theorem~\ref{prop:converse},
$e^{x_1},\ldots,e^{x_n}$ are algebraically independent over
$\mathbb{Q}$. Hence
$
\operatorname{trdeg}_{\mathbb{Q}}
\mathbb{Q}(e^{x_1},\ldots,e^{x_n})=n.
$
Since $
\mathbb{Q}(e^{x_1},\ldots,e^{x_n})
\subseteq
\mathbb{Q}(x_1,\ldots,x_n,e^{x_1},\ldots,e^{x_n}),
$
the claimed inequality follows, which is exactly the conclusion of  Schanuel's conjecture. 
\end{proof}

\section{Experiments}

%
\begin{figure*}[t]
\centering
\begin{tikzpicture}
\begin{groupplot}[
  group style={group size=3 by 1, horizontal sep=1.4em},
  width=0.252\textwidth,
  height=0.18\textwidth,
  scale only axis,
  xmin=0.5, xmax=10.5,
  ymin=0.5, ymax=6.5,
  y dir=reverse,
  xtick={1,...,10},
  ytick={1,2,3,4,5,6},
  yticklabels={sigmoid,tanh,sinh,cosh,exp,ReLU},
  xlabel={Aggregation-weight $A$},
  tick label style={font=\scriptsize},
  x tick label style={font=\tiny,rotate=90,anchor=east},
  label style={font=\small},
  title style={font=\small},
  point meta min=0,
  point meta max=1,
  colormap={inferno}{
    rgb255(0cm)=(252,255,164);
    rgb255(1cm)=(246,215,70);
    rgb255(2cm)=(252,165,10);
    rgb255(3cm)=(243,120,25);
    rgb255(4cm)=(221,81,58);
    rgb255(5cm)=(188,55,84);
    rgb255(6cm)=(147,38,103);
    rgb255(7cm)=(106,23,110);
    rgb255(8cm)=(66,10,104);
    rgb255(9cm)=(22,11,57);
    rgb255(10cm)=(0,0,4)
  },
]

\nextgroupplot[
  title={Self-weight $C=0.01$},
  xticklabels={0.01,0.02,0.03,0.04,0.05,0.06,0.07,0.08,0.09,0.10},
]
\addplot[
  matrix plot*,
  mesh/cols=10,
  point meta=explicit,
  visualization depends on={\thisrow{z} \as \cellvalue},
  nodes near coords={%
    \pgfmathparse{equal(\cellvalue,1)}%
    \ifnum\pgfmathresult=1\relax
      $\checkmark$%
    \fi
  },
  every node near coord/.append style={
    text=white,
    font=\scriptsize,
    anchor=center,
  },
  draw=white,
  line width=0.15pt,
] table[meta=z] {
x y z
1 1 0.342
2 1 1.000
3 1 1.000
4 1 1.000
5 1 1.000
6 1 1.000
7 1 1.000
8 1 1.000
9 1 1.000
10 1 1.000
1 2 0.352
2 2 1.000
3 2 1.000
4 2 1.000
5 2 1.000
6 2 1.000
7 2 0.998
8 2 1.000
9 2 1.000
10 2 1.000
1 3 0.358
2 3 0.998
3 3 0.998
4 3 0.998
5 3 1.000
6 3 0.998
7 3 1.000
8 3 1.000
9 3 1.000
10 3 1.000
1 4 0.336
2 4 0.984
3 4 0.984
4 4 0.984
5 4 0.984
6 4 0.984
7 4 0.984
8 4 0.984
9 4 0.984
10 4 0.984
1 5 0.342
2 5 1.000
3 5 1.000
4 5 1.000
5 5 1.000
6 5 1.000
7 5 1.000
8 5 1.000
9 5 0.998
10 5 1.000
1 6 0.022
2 6 0.232
3 6 0.424
4 6 0.472
5 6 0.686
6 6 0.574
7 6 0.646
8 6 0.636
9 6 0.784
10 6 0.706
};

\nextgroupplot[
  title={Self-weight $C=0.1$},
  xticklabels={0.1,0.2,0.3,0.4,0.5,0.6,0.7,0.8,0.9,1.0},
  yticklabels={},
]
\addplot[
  matrix plot*,
  mesh/cols=10,
  point meta=explicit,
  visualization depends on={\thisrow{z} \as \cellvalue},
  nodes near coords={%
    \pgfmathparse{equal(\cellvalue,1)}%
    \ifnum\pgfmathresult=1\relax
      $\checkmark$%
    \fi
  },
  every node near coord/.append style={
    text=white,
    font=\scriptsize,
    anchor=center,
  },
  draw=white,
  line width=0.15pt,
] table[meta=z] {
x y z
1 1 0.346
2 1 1.000
3 1 1.000
4 1 0.998
5 1 0.998
6 1 0.998
7 1 1.000
8 1 1.000
9 1 1.000
10 1 1.000
1 2 0.358
2 2 1.000
3 2 0.998
4 2 1.000
5 2 1.000
6 2 1.000
7 2 1.000
8 2 1.000
9 2 1.000
10 2 1.000
1 3 0.358
2 3 1.000
3 3 1.000
4 3 1.000
5 3 1.000
6 3 0.992
7 3 0.988
8 3 0.976
9 3 0.942
10 3 0.886
1 4 0.344
2 4 1.000
3 4 0.990
4 4 0.936
5 4 0.790
6 4 0.738
7 4 0.732
8 4 0.724
9 4 0.686
10 4 0.594
1 5 0.348
2 5 0.986
3 5 0.838
4 5 0.738
5 5 0.732
6 5 0.722
7 5 0.650
8 5 0.510
9 5 0.302
10 5 0.146
1 6 0.014
2 6 0.234
3 6 0.386
4 6 0.506
5 6 0.702
6 6 0.586
7 6 0.690
8 6 0.628
9 6 0.776
10 6 0.660
};

\nextgroupplot[
  title={Self-weight $C=1$},
  xticklabels={1,2,3,4,5,6,7,8,9,10},
  yticklabels={},
  colorbar,
  colorbar style={
    y dir=normal,
    ytick={0,0.5,1},
    yticklabel style={font=\scriptsize},
    width=0.16cm,
  },
]
\addplot[
  matrix plot*,
  mesh/cols=10,
  point meta=explicit,
  visualization depends on={\thisrow{z} \as \cellvalue},
  nodes near coords={%
    \pgfmathparse{equal(\cellvalue,1)}%
    \ifnum\pgfmathresult=1\relax
      $\checkmark$%
    \fi
  },
  every node near coord/.append style={
    text=white,
    font=\scriptsize,
    anchor=center,
  },
  draw=white,
  line width=0.15pt,
] table[meta=z] {
x y z
1 1 0.342
2 1 1.000
3 1 1.000
4 1 0.992
5 1 0.868
6 1 0.546
7 1 0.224
8 1 0.078
9 1 0.040
10 1 0.032
1 2 0.342
2 2 1.000
3 2 0.818
4 2 0.318
5 2 0.064
6 2 0.026
7 2 0.018
8 2 0.018
9 2 0.018
10 2 0.018
1 3 0.286
2 3 0.728
3 3 0.536
4 3 0.132
5 3 0.032
6 3 0.022
7 3 0.018
8 3 0.018
9 3 0.018
10 3 0.018
1 4 0.110
2 4 0.018
3 4 0.018
4 4 0.018
5 4 0.018
6 4 0.018
7 4 0.018
8 4 0.018
9 4 0.018
10 4 0.018
1 5 0.008
2 5 0.018
3 5 0.018
4 5 0.018
5 5 0.018
6 5 0.018
7 5 0.018
8 5 0.018
9 5 0.018
10 5 0.018
1 6 0.006
2 6 0.118
3 6 0.242
4 6 0.356
5 6 0.610
6 6 0.404
7 6 0.524
8 6 0.514
9 6 0.712
10 6 0.596
};
\end{groupplot}
\end{tikzpicture}
\caption{Success ratio in float64 arithmetic 
for GNNs with activation (row) and aggregation-weight (column)
at fixed self-weights in three weight categories.
Checkmarks denote success ratio
$1$.}
\vspace*{-1mm}
\label{fig:float64-fixed-c-heatmaps}
\end{figure*}
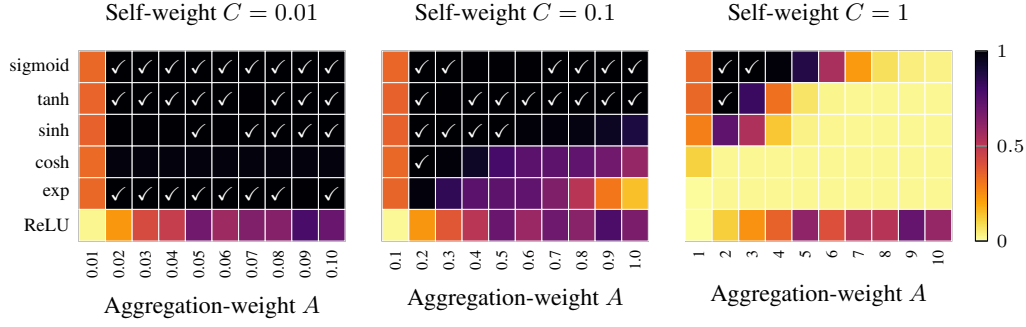

We empirically investigate the vertex distinguishing power of GNNs using  sigmoid, tanh, sinh, cosh and exp as prominent Schanuel-ready activation functions. In addition, we evaluate ReLU as a popular activation function that is not Schanuel-ready. 
Since Schanuel-ready activation functions may compute transcendental numbers from rational ones, we 
vary the 
precision and use
32-, 64- and 128-bit floating point arithmetics.
We  randomly generate $N=500$
Erd\H{o}s--R\'enyi graphs with $n \in [20, 100]$ vertices using 
edge probability\footnote{We avoid graphs generated by a fixed edge probability as in this setting 1-WL is known to, with high probability, attribute distinct colours to all vertices in a graph; see Table~1 of \cite{labeling_sparse_graphs}.} $\frac{2}{n}$, with  one-hot
feature vectors from $\mathbb{Q}^4$. 
On each graph, we run 1-WL. If $L$ 
rounds are needed for stabilization, we  construct $L$-layer GNNs
$(\sigma,(C_i)_{1 \leq i \leq L},(A_i)_{1 \leq i \leq L},(b_i)_{1 \leq i \leq L},I)$ where 
$b_1 = \dots = b_L = 0$ and
$C_1 = \dots = C_L$ as well as $A_1 = \dots = A_L$
range over $\{0.01, 0.02, \dots, 0.1\}$, $\{0.1, 0.2, \dots, 1\}$ and
$\{1, 2, \dots, 10\}$.
We choose $I$ so that the initialisation function is
$\widetilde{\lambda}^{(0)}(x) = \sigma((2^{-1}, 2^{-2}, 2^{-3}, 2^{-4}) \times x)$.
Say that $\gnn$ \emph{succeeds} on a graph $G=(V,E,\lambda)$ if for all $u, v \in V$,
$\mu^{(L)}(v) = \mu^{(L)}(u) $ if and only if $\lambda^{(L)}(v) = \lambda^{(L)}(u)$.
The \emph{success ratio} of $\gnn$ is  the number of graphs that $\gnn$ succeeds on divided by $N$. 
%

An excerpt of our results for float64 is displayed in Figure~\ref{fig:float64-fixed-c-heatmaps}, see the appendix for 32- and 128-bit results.
For many weight settings, GNNs with Schanuel-ready activation functions perfectly simulate
the 1-WL algorithm. The  
exception of $\frac{C}{A} \in \mathbb{N}$ predicted by our theoretical analysis is observable in the leftmost columns.
All  Schanuel-ready activation functions encounter a performance degradation when weights get 
larger, which we believe is because
computed values get exceedingly close to 1 and cannot be distinguished by a fixed-precision floating-point representation (for sigmoid and tanh), or overflow
(for sinh, cosh, and exp).
These different phenomena can also be used to explain why the performance of sinh, cosh, and exp degrades faster than that of sigmoid and tanh (see appendix). 
Regardless of which Schanuel-ready activation function is used,
higher precision floats result in better performance.
As expected from the results in \cite{KhalifePolynomialActivationResult}, ReLU GNNs exhibit weaker performance than GNNs with Schanuel-ready activation functions.

\section{Conclusions}

We have studied the capability of GNNs with constant depth and  rational weights to match the expressive power of 1-WL,
establishing an intimate connection to algebraic independence and, more specifically, Schanuel's conjecture. As future work, we would be interested in understanding whether the assumptions made in Section~\ref{section_converse_direction} can be weakened. 
Can the same result be proved for sigmoid or all Schanuel-ready activation functions?
Is it necessary to assume the presence of virtual nodes and that the embedding
computed by the final GNN layer is of dimension~1? Can the restricted form of Schanuel's conjecture
obtained in Corollary~\ref{cor:restricted-schanuel} be strengthened, possibly even to Schanuel's conjecture without any
restrictions?
We conjecture that our methods can be extended to more expressive
GNN architectures and their natural WL test counterparts;
e.g. $k$-GNNs and the $k$-WL test, GNNs with
global readout in each layer and 1-WL with global aggregation.
Such directions are left for future research.

\section{Acknowledgements}
Daumantas Kojelis and Carsten Lutz were
supported by DFG project LU 1417/4-1.
Additionally, Carsten Lutz was partly supported by BMFTR in DAAD project 57616814 (SECAI).

The authors also wish to thank Valentin Wolfer for assisting with the experiments.

\section*{AI Use Statement}
The generative AI model ChatGPT-5.6 Sol (Max)
was used as an aid when developing this paper.
The extent of the model's involvement is as follows.
All text (including the appendix) has been human written
with occasional paraphrasing requests relegated to AI. 
All mathematical claims and statements have been formulated
by humans without help of AI.
All proofs excluding Claims 14, 15 and 21 were first obtained
by humans without AI help. Claims 14, 15 and 21 have
been realised by strengthening our initial human-generated proofs
with AI guidance.
Completed proofs were given to the AI agent in order to identify
gaps and inaccuracies.
Some suggestions arising from the above have been implemented.
Figures have been generated by AI.
Some references have been recommended by AI agents.
Experimental code has been written
with the help of an AI coding agent.
The authors have verified the AI-generated figures,
proof suggestions, code and references, and 
assume responsibility for all contents of this work.

\section*{Reproducibility Statement}
Since our theoretical results rely on Schanuel's conjecture, we explicitly indicated this dependence in claim, lemma, proposition and theorem statements. The appendix houses full proofs and additional experimental results omitted from the main body of text. We provide code for running our experiments in
the supplementary material with documentation.

\section*{Ethics Statement}
Our work aims to increase understanding in
the expressive power of GNNs. Our results are 
theoretic in nature. We therefore believe that this work
is low risk in terms of possible misuse.

\bibliography{references}

@inproceedings{xu2018howPowerful,
title={How Powerful are Graph Neural Networks?},
author={Keyulu Xu and Weihua Hu and Jure Leskovec and Stefanie Jegelka},
booktitle={International Conference on Learning Representations},
year={2019},
url={https://openreview.net/forum?id=ryGs6iA5Km},
}

@inproceedings{WLGoesNeural,
  author       = {Christopher Morris and
                  Martin Ritzert and
                  Matthias Fey and
                  William L. Hamilton and
                  Jan Eric Lenssen and
                  Gaurav Rattan and
                  Martin Grohe},
  title        = {Weisfeiler and Leman Go Neural: Higher-Order Graph Neural Networks},
  booktitle    = {Proceedings of the Thirty-Third {AAAI} Conference on Artificial Intelligence, {AAAI}
                  2019},
  pages        = {4602--4609},
  publisher    = {{AAAI} Press},
  year         = {2019},
  url          = {https://doi.org/10.1609/aaai.v33i01.33014602},
  doi          = {10.1609/AAAI.V33I01.33014602},
  bibsource    = {dblp computer science bibliography, https://dblp.org}
}

@inproceedings{provablyPowerfulGNNs,
 author = {Maron, Haggai and Ben-Hamu, Heli and Serviansky, Hadar and Lipman, Yaron},
 booktitle = {Advances in Neural Information Processing Systems},
 editor = {H. Wallach and H. Larochelle and A. Beygelzimer and F. d\textquotesingle Alch\'{e}-Buc and E. Fox and R. Garnett},
 pages = {},
 publisher = {Curran Associates, Inc.},
 title = {Provably Powerful Graph Networks},
 url = {https://proceedings.neurips.cc/paper_files/paper/2019/file/bb04af0f7ecaee4aae62035497da1387-Paper.pdf},
 volume = {32},
 year = {2019}
}

@inproceedings{randomNodeInit,
  title     = {The Surprising Power of Graph Neural Networks with Random Node Initialization},
  author    = {Abboud, Ralph and Ceylan, İsmail İlkan and Grohe, Martin and Lukasiewicz, Thomas},
  booktitle = {Proceedings of the Thirtieth International Joint Conference on
               Artificial Intelligence, {IJCAI-21}},
  publisher = {International Joint Conferences on Artificial Intelligence Organization},
  editor    = {Zhi-Hua Zhou},
  pages     = {2112--2118},
  year      = {2021},
  month     = {8},
  note      = {Main Track},
  doi       = {10.24963/ijcai.2021/291},
  url       = {https://doi.org/10.24963/ijcai.2021/291},
}

@INPROCEEDINGS{GroheLogic,
  author={Grohe, Martin},
  booktitle={2021 36th Annual ACM/IEEE Symposium on Logic in Computer Science (LICS)}, 
  title={The Logic of Graph Neural Networks}, 
  year={2021},
  volume={},
  number={},
  pages={1-17},
  doi={10.1109/LICS52264.2021.9470677}}

@inproceedings{GroheRepetition,
  author       = {Eran Rosenbluth and
                  Martin Grohe},
  editor       = {Sven Koenig and
                  Chad Jenkins and
                  Matthew E. Taylor},
  title        = {Repetition Makes Perfect: Recurrent Graph Neural Networks Match Message
                  Passing Limit},
  booktitle    = {Fortieth {AAAI} Conference on Artificial Intelligence, Thirty-Eighth
                  Conference on Innovative Applications of Artificial Intelligence,
                  Sixteenth Symposium on Educational Advances in Artificial Intelligence,
                  {AAAI} 2026, Singapore, January 20-27, 2026},
  pages        = {25168--25175},
  publisher    = {{AAAI} Press},
  year         = {2026},
  url          = {https://doi.org/10.1609/aaai.v40i30.39707},
  doi          = {10.1609/AAAI.V40I30.39707},
  bibsource    = {dblp computer science bibliography, https://dblp.org}
}

@inproceedings{WLTestReLULogarithmicResult,
 author = {Aamand, Anders and Chen, Justin and Indyk, Piotr and Narayanan, Shyam and Rubinfeld, Ronitt and Schiefer, Nicholas and Silwal, Sandeep and Wagner, Tal},
 booktitle = {Advances in Neural Information Processing Systems},
 doi = {10.52202/068431-1982},
 editor = {S. Koyejo and S. Mohamed and A. Agarwal and D. Belgrave and K. Cho and A. Oh},
 pages = {27333--27346},
 publisher = {Curran Associates, Inc.},
 title = {Exponentially Improving the Complexity of Simulating the Weisfeiler-Lehman Test with Graph Neural Networks},
 url = {https://proceedings.neurips.cc/paper_files/paper/2022/file/af0ad514b9cda46bd49e14ee11e2672f-Paper-Conference.pdf},
 volume = {35},
 year = {2022}
}

@InProceedings{bravoWLAnalyticalResult,
  title = 	 {On dimensionality of feature vectors in {MPNN}s},
  author =       {Bravo, C\'{e}sar and Kozachinskiy, Alexander and Rojas, Cristobal},
  booktitle = 	 {Proceedings of the 41st International Conference on Machine Learning},
  pages = 	 {4472--4481},
  year = 	 {2024},
  editor = 	 {Salakhutdinov, Ruslan and Kolter, Zico and Heller, Katherine and Weller, Adrian and Oliver, Nuria and Scarlett, Jonathan and Berkenkamp, Felix},
  volume = 	 {235},
  series = 	 {Proceedings of Machine Learning Research},
  month = 	 {21--27 Jul},
  publisher =    {PMLR},
  url = 	 {https://proceedings.mlr.press/v235/bravo24a.html},
}

@inproceedings{OtherWLAnalyticalResult,
 author = {Amir, Tal and Gortler, Steven and Avni, Ilai and Ravina, Ravina and Dym, Nadav},
 booktitle = {Advances in Neural Information Processing Systems},
 doi = {10.52202/075280-1845},
 editor = {A. Oh and T. Naumann and A. Globerson and K. Saenko and M. Hardt and S. Levine},
 pages = {42516--42551},
 publisher = {Curran Associates, Inc.},
 title = {Neural Injective Functions for Multisets, Measures and Graphs via a Finite Witness Theorem},
 url = {https://proceedings.neurips.cc/paper_files/paper/2023/file/84b686f7cc7b7751e9aaac0da74f755a-Paper-Conference.pdf},
 volume = {36},
 year = {2023}
}

@article{KhalifePolynomialActivationResult,
author = {Khalife, Sammy and Basu, Amitabh},
title = {On the Power of Graph Neural Networks and the Role of the Activation Function},
journal = {SIAM Journal on Applied Algebra and Geometry},
volume = {10},
number = {2},
pages = {350-365},
year = {2026},
doi = {10.1137/24M1659704},
URL = {https://doi.org/10.1137/24M1659704},
eprint = {https://doi.org/10.1137/24M1659704}
}

@inproceedings{BarceloPrecomputedHom,
 author = {Barcel\'{o}, Pablo and Geerts, Floris and Reutter, Juan and Ryschkov, Maksimilian},
 booktitle = {Advances in Neural Information Processing Systems},
 editor = {M. Ranzato and A. Beygelzimer and Y. Dauphin and P.S. Liang and J. Wortman Vaughan},
 pages = {25280--25293},
 publisher = {Curran Associates, Inc.},
 title = {Graph Neural Networks with Local Graph  Parameters},
 url = {https://proceedings.neurips.cc/paper_files/paper/2021/file/d4d8d1ac7e00e9105775a6b660dd3cbb-Paper.pdf},
 volume = {34},
 year = {2021}
}

@inproceedings{originalDHNpaper,
 author = {Maehara, Takanori and NT, Hoang},
 booktitle = {Advances in Neural Information Processing Systems},
 doi = {10.52202/079017-1784},
 editor = {A. Globerson and L. Mackey and D. Belgrave and A. Fan and U. Paquet and J. Tomczak and C. Zhang},
 pages = {56076--56107},
 publisher = {Curran Associates, Inc.},
 title = {Deep Homomorphism Networks},
 url = {https://proceedings.neurips.cc/paper_files/paper/2024/file/65f54fdf62cd5614dc5715ae7ece4ef6-Paper-Conference.pdf},
 volume = {37},
 year = {2024}
}

@inproceedings{GNNsWithDistanceEnc,
 author = {Li, Pan and Wang, Yanbang and Wang, Hongwei and Leskovec, Jure},
 booktitle = {Advances in Neural Information Processing Systems},
 editor = {H. Larochelle and M. Ranzato and R. Hadsell and M.F. Balcan and H. Lin},
 pages = {4465--4478},
 publisher = {Curran Associates, Inc.},
 title = {Distance Encoding: Design Provably More Powerful Neural Networks for Graph Representation Learning},
 url = {https://proceedings.neurips.cc/paper_files/paper/2020/file/2f73168bf3656f697507752ec592c437-Paper.pdf},
 volume = {33},
 year = {2020}
}

@inproceedings{gilmer2017neural,
  title     = {Neural Message Passing for Quantum Chemistry},
  author    = {Gilmer, Justin and Schoenholz, Samuel S. and Riley, Patrick F.
               and Vinyals, Oriol and Dahl, George E.},
  booktitle = {Proceedings of the 34th International Conference on Machine Learning},
  pages     = {1263--1272},
  year      = {2017},
  volume    = {70},
  series    = {Proceedings of Machine Learning Research},
  publisher = {PMLR}
}

@inproceedings{hu2020open,
  title     = {Open Graph Benchmark: Datasets for Machine Learning on Graphs},
  author    = {Hu, Weihua and Fey, Matthias and Zitnik, Marinka and Dong, Yuxiao
               and Ren, Hongyu and Liu, Bowen and Catasta, Michele
               and Leskovec, Jure},
  booktitle = {Advances in Neural Information Processing Systems},
  volume    = {33},
  pages     = {22118--22133},
  year      = {2020}
}

@inproceedings{trainingSigmoidNNs,
  author       = {Teemu Hankala and
                  Miika Hannula and
                  Juha Kontinen and
                  Jonni Virtema},
  editor       = {Michael J. Wooldridge and
                  Jennifer G. Dy and
                  Sriraam Natarajan},
  title        = {Complexity of Neural Network Training and {ETR:} Extensions with Effectively
                  Continuous Functions},
  booktitle    = {Thirty-Eighth {AAAI} Conference on Artificial Intelligence, {AAAI}
                  2024, Thirty-Sixth Conference on Innovative Applications of Artificial
                  Intelligence, {IAAI} 2024, Fourteenth Symposium on Educational Advances
                  in Artificial Intelligence, {EAAI} 2024, February 20-27, 2024, Vancouver,
                  Canada},
  pages        = {12278--12285},
  publisher    = {{AAAI} Press},
  year         = {2024},
  url          = {https://doi.org/10.1609/aaai.v38i11.29118},
  doi          = {10.1609/AAAI.V38I11.29118},
  bibsource    = {dblp computer science bibliography, https://dblp.org}
}

@InProceedings{DNNsVerificationSigmoid,
  author =	{Isac, Omri and Zohar, Yoni and Barrett, Clark and Katz, Guy},
  title =	{{DNN Verification, Reachability, and the Exponential Function Problem}},
  booktitle =	{34th International Conference on Concurrency Theory (CONCUR 2023)},
  pages =	{26:1--26:18},
  series =	{Leibniz International Proceedings in Informatics (LIPIcs)},
  ISBN =	{978-3-95977-299-0},
  ISSN =	{1868-8969},
  year =	{2023},
  volume =	{279},
  editor =	{P\'{e}rez, Guillermo A. and Raskin, Jean-Fran\c{c}ois},
  publisher =	{Schloss Dagstuhl -- Leibniz-Zentrum f{\"u}r Informatik},
  address =	{Dagstuhl, Germany},
  URL =		{https://drops.dagstuhl.de/entities/document/10.4230/LIPIcs.CONCUR.2023.26},
  URN =		{urn:nbn:de:0030-drops-190205},
  doi =		{10.4230/LIPIcs.CONCUR.2023.26}
}

@book{Lang1966TranscendentalNumbers,
  author    = {Lang, Serge},
  title     = {Introduction to Transcendental Numbers},
  publisher = {Addison-Wesley},
  address   = {Reading, MA},
  year      = {1966}
}

@article{Weierstrass1885Lindemann,
  author  = {Weierstrass, Karl},
  title   = {Zu {Lindemann}'s {Abhandlung}: {``}{\"U}ber die {Ludolph}'sche {Zahl}{''}},
  journal = {Sitzungsberichte der K{\"o}niglich Preussischen Akademie der Wissenschaften zu Berlin},
  year    = {1885},
  pages   = {1067--1085}
}

@book{CladerRoss2025,
  author    = {Emily Clader and Dustin Ross},
  title     = {Beginning in Algebraic Geometry},
  series    = {Undergraduate Texts in Mathematics},
  publisher = {Springer},
  year      = {2025},
  doi       = {10.1007/978-3-031-88819-9},
  isbn      = {978-3-031-88819-9}
}

@inproceedings{khalifeIsUniformExpressiveness,
title={Is uniform expressivity too restrictive? Towards efficient expressivity of {GNN}s},
author={Sammy Khalife and Josu{\'e} Tonelli-Cueto},
booktitle={The Thirteenth International Conference on Learning Representations},
year={2025},
url={https://openreview.net/forum?id=lsvGqR6OTf}
}

@InProceedings{labeling_sparse_graphs,
  author =	{Verbitsky, Oleg and Zhukovskii, Maksim},
  title =	{Canonical Labeling of Sparse Random Graphs},
  booktitle =	{42nd International Symposium on Theoretical Aspects of Computer Science (STACS 2025)},
  pages =	{75:1--75:20},
  series =	{Leibniz International Proceedings in Informatics (LIPIcs)},
  ISBN =	{978-3-95977-365-2},
  ISSN =	{1868-8969},
  year =	{2025},
  volume =	{327},
  publisher =	{Schloss Dagstuhl -- Leibniz-Zentrum f{\"u}r Informatik},
  address =	{Dagstuhl, Germany},
  URL =		{https://drops.dagstuhl.de/entities/document/10.4230/LIPIcs.STACS.2025.75},
  URN =		{urn:nbn:de:0030-drops-229003},
  doi =		{10.4230/LIPIcs.STACS.2025.75}
}
\bibliographystyle{iclr2027_conference}

\newpage 
\appendix
\section*{Appendix}
\section{Full proof of Proposition~\ref{main_schanuels_proposition}}

Let us recall the GNNs we are working with.
Recall that an activation function $\sigma$ is Schanuel-ready
if there exist polynomials $P(y)$ and $Q(y)$ with rational coefficients such that
\begin{itemize}
    \item $\sigma(x) = \frac{P(e^x)}{Q(e^x)}$,
    \item $\frac{P(y)}{Q(y)}$ is non-constant, positive and strictly increasing on $y > 1$,
    \item $Q(y) \neq 0$ for all $y > 0$.
\end{itemize}
When $\sigma$ is clear from context we will often refer to $\frac{P(y)}{Q(y)}$ as $R(y)$.

Now, writing $N$ for the dimensionality of initial feature vectors recall that, in our full setting, an $L$-layered GNN $\gnn$ is Schanuel-ready if 
\begin{itemize}
    \item each layer uses a common Schanuel-ready activation function $\sigma$,
    \item initialisation takes the form $\lambda^{(0)}(v) = \sigma((2^1, 2^2,\dots, 2^{N}) \times \sigma(\lambda(v) + \pmb{1}_N))$,
    \item each layer computes $\lambda^{(i)}(v) = \sigma(C_i\lambda^{(i-1)}(v) + A_i\sum_{u \in N(v)} \lambda^{(i-1)}(u) + b_i)$,
    where $A_i, C_i > 0$, $b_i \geq 0$ are rational numbers such that $\frac{C_i}{A_i} \not\in \mathbb{N}$. 
\end{itemize}
We again note that a GNN having $A_1 = \dots = A_L = 2$, $C_1 = \dots = C_L = 1$ and $b_1 = \dots = b_L = 0$
is a valid choice for our purposes.
Notice that each $\lambda^{(0)}, \dots, \lambda^{(L)}$ produces 1-dimensional feature vectors, but $\lambda$ -- the feature map of any given graph -- 
need not be so. We now prove the following in its full form:
\mainSchanuelsProposition*


The overall proof structure follows the one in the main body. The setting, however,
is more technical as we aim to accommodate a wider range of activation functions
and the scope of graphs with non-negative rational features.
Thus, recall that $\sigma(x) = \frac{P(e^x)}{Q(e^x)}$ is non-constant,
where $P, Q$ are polynomials over $\mathbb{Q}$ with $Q(x) \neq 0$ for all $x > 0$.

In the sequel we will make heavy use of a series of constants $\kappa_1, \kappa_2, \dots \in \mathbb{N}$,
which are defined as follows.
Enumerate positive valued rationals in any way: $l_1, l_2 \dots$ and recall that all
weights of our GNN $A_{1}, C_{1}, b_1, \dots, A_{L}, C_{L}, b_{L}$
are positive (non-negative in the case of $b_1, \dots, b_L$) rationals.
For a given $k \geq 1$ the constant $\kappa_k$ is the least common denominator of $l_1, \dots, l_k$
as well as $A_{1}, C_{1}, b_1, \dots, A_{L}, C_{L}, b_{L}$.
When $k$ is clear from context, we simply write $\kappa = \kappa_k$.

With the above at hand we take any $k \geq 1$ and $r \geq 1$ and define the following sets inductively on $i < L$:
\begin{itemize}
    \item $\widehat{Z}_{-1, k, r} = \{ e^{\frac{1}{\kappa}} \}$,
    \item $Z_{-1, k, r} = \{ \sigma(l_1), \dots, \sigma(l_k) \} =
        \{R((e^{\frac{1}{\kappa}})^{l_j\kappa}) \mid 1 \leq j \leq k \}$,
    \item $\widehat{Z}_{i+1, k, r}  = \{ e^\frac{z}{\kappa} \mid z \in Z_{i, k, r} \}$, and
    \item $Z_{i+1, k, r} = \{R(z_1 \cdot {\dots} \cdot z_r) \mid
        z_1, \dots, z_r \in \widehat{Z}_{i+1, k, r} \cup \{ 1, e^{\frac{1}{\kappa}} \}
    \text{ not all in } \{1, e^{\frac{1}{\kappa}} \}\}$.
\end{itemize}
The definitions of $\widehat{Z}_{-1, k, r}$ and $Z_{-1, k, r}$ are technical and serve as an intermediary step
for computing the initialisation $\lambda^{(0)}$.
To aid intuition, we invite the reader to think of $Z_{0, k, r}$ as a $k$-sized sample of $\lambda^{(0)}$'s initialisation values;
it being understood that $r$ is ``large enough''.
The set $Z_{1, k, r}$ is then again a sample of values that $\gnn$ can compute by, in a sense, aggregating and combining values of $Z_{0, k, r}$.
Subsequent sets $Z_{i, k, r}$ then behave similarly. 
The intuition behind $\kappa$ is that raising $e^{\frac{1}{\kappa}}$ to appropriate powers
will ultimately recover (exponentiations of) $l_1, \dots, l_k$ as well as $A_{1}, C_{1}, b_1, \dots, A_{L}, C_{L}, b_{L}$.
We show how this is achieved by writing $Z_{i, k} = \bigcup_{r \geq 1} Z_{i, k, r}$ and:
\begin{claim}\label{claim_appendix_computation_in_Z_sets}
    For all featured graphs $G$, all vertices $v \in V$ and all $0 \leq i \leq L$ there is some $k' \geq 1$
    such that for all $k \geq k'$ we have $\lambda^{(i)}(v) \in Z_{i, k}$.
\end{claim}
\begin{proof}
    Take any graph $G$ and
    any vertex $v \in V$. Recall that $\lambda(v)$ is a
    vector of non-negative rationals, say, $m_1, \dots, m_N$, it being
    understood that $N = |\lambda(v)|$.
    Since each entry is non-negative, we may find $l_{k_1}, \dots, l_{k_N}$ in our
    enumeration of the positive rational numbers, such that
    $l_{k_t} = m_t + 1$. Now, let $k'_v = \max(k_1, \dots, k_N)$.
    Clearly, $\sigma(l_{k_t}) \in Z_{-1, k'_v}$ thus resulting in $\sigma(m_t + 1) \in Z_{-1, k'_v}$ for all $1 \leq t \leq N$.
    Now, taking $k' = \max\{k'_v \mid v \in V\}$ we have, for each $k \geq k'$,
    that every entry $m$ in any feature vector of vertices in $G$ satisfies $\sigma(m+1) \in Z_{-1, k}$.
    
    Taking any $k$ as above and writing $\kappa = \kappa_k$ we proceed by induction on $0 \leq i < L$.
    For the base case take any $v$ and let $m_{1}, \dots, m_{N}$ be the entries of $\lambda(v)$ in order.
    As argued above, $\sigma(m_{1} + 1), \dots, \sigma(m_{N} + 1) \in Z_{-1, k}$ and thus
    $e^{\frac{\sigma(m_{1} + 1)}{\kappa}}, \dots, e^{\frac{\sigma(m_{N} + 1)}{\kappa}} \in \widehat{Z}_{0, k}$.
    But then, again by definition,
    $R(\prod_{j=1}^n e^{2^j\kappa \cdot \frac{\sigma(m_j+1)}{\kappa}})
    = \sigma((2^1, 2^2, \dots, 2^N) \times \sigma(\lambda(v) + \pmb{1}_N))
    = \lambda^{(0)}(v)$
    is in $Z_{0, k, r}$ for $r$ large enough, and thus also in $Z_{0, k}$.
    
    For the inductive step assume that $\lambda^{(i)}(v) \in Z_{i, k}$ for all $v \in V$.
    Then, aggregation and combination in $\gnn$ amount to
    $\sigma\big(C_{i+1} z_0 + A_{i+1} (z_1 + \dots + z_n) + b_{i+1} \big)$, where $z_0, \dots, z_n \in Z_{i, k}$.
    By the definition of $\sigma$ we have that the above is
    $R(e^{C_{i+1} z_0} \cdot \prod_{j=1}^n e^{A_{i+1}z_j} \cdot e^{b_{i+1}})$.
    Noting that $C_{i+1} = \frac{C_{i+1}'}{\kappa}$, $A_{i+1} = \frac{A_{i+1}'}{\kappa}$, and
    $b_{i+1} = \frac{b_{i+1}'}{\kappa}$,
    for some $C_{i+1}', A_{i+1}', b_{i+1}' \in \mathbb{N}$ we have that
    $e^{C_{i+1}z_0} = (e^{\frac{z_0}{\kappa}})^{C_{i+1}'}$ with $e^{\frac{z_0}{\kappa}}$
    being a member of $\widehat{Z}_{i+1, k}$.
    Similar observations hold for $e^{A_{i+1}z_j}$ with $1 \leq j \leq n$.
    Lastly, recalling that multiples of $e^{\frac{1}{\kappa}}$
    are allowed to appear when forming $Z_{i+1, k}$, we have $e^{b_{i+1}} = (e^{\frac{1}{\kappa}})^{b_{i+1}'}$.
    Then, clearly, 
    $R(e^{C_{i+1} z_0} \cdot \prod_{j=1}^n (e^{A_{i+1}z_j}) \cdot e^{b_{i+1}}) = 
    R((e^{\frac{z_0}{\kappa}})^{C_{i+1}'} \cdot \prod_{j=1}^n (e^{\frac{z_j}{\kappa}})^{A_{i+1}'} \cdot
    (e^{\frac{1}{\kappa}})^{b_{i+1}'})$
    which is then in $Z_{i+1, k, r}$ for large enough $r$, and thus also in $Z_{i+1, k}$.
\end{proof}

Our strategy again consists of establishing linear independence of $Z_{i, k}$ for all $0 \leq i \leq L$ and $k \geq 1$
in order to conclude that $\gnn$ matches the $L$-round 1-WL test. We fully prove this claim at the end of this section.
The next few claims are a setup for the desired conclusion.
\begin{claim}\label{claim_limit_for_lin_indep}
    There is some $t \in \mathbb{Q}$ and non-zero $r \in \mathbb{Z}$ such that
    $\alpha = \lim_{x \to 0^+} \frac{R(x) - t}{x^r}$
    is non-zero.
\end{claim}
\begin{proof}
    Since $P(x)$ and $Q(x)$ are non-zero polynomials
    we may find $p, q \in \mathbb{N}$ such that 
    $P(x) = x^pP_0(x)$ and $Q(x) = x^qQ_0(x)$
    for non-zero polynomials $P_0(x), Q_0(x)$ both having a non-zero constant term (and thus $P_0(0) \neq 0 \neq Q_0(0)$).
    Suppose first that $p \neq q$. Then, by taking $t = 0$ and $r = p-q$
    we have
    \begin{equation*}
        \alpha = \lim_{x \to 0^+} \frac{\frac{P(x)}{Q(x)} - t}{x^r} = \lim_{x \to 0^+} \frac{x^{p-q} \cdot \frac{P_0(x)}{Q_0(x)}}{x^{p-q}}
        = \lim_{x\to 0^+} \frac{P_0(x)}{Q_0(x)} \neq 0.
    \end{equation*}

    Now suppose that $p = q$. Then, taking $P_0$ and $Q_0$ as above we again have
    that $\frac{P_0(0)}{Q_0(0)} \neq 0$.
    Now, let us take a look at
    $S(x) = P_0(x)Q_0(0) - P_0(0)Q_0(x)$. Notice that $S(x)$ is not the zero
    polynomial as we would then have $\frac{P_0(x)}{Q_0(x)} = \frac{P_0(0)}{Q_0(0)}$ for all $x > 0$.
    Inspecting $S(x)$ we see that the constant terms of polynomials $P_0(x)Q_0(0)$
    and $P_0(0)Q_0(x)$ cancel each other out. Thus we may assume that $S(x)$ takes the normal form
    $a_0 + a_1x^1 + \dots + a_nx^n$ with $a_0 = 0$.
    Now, let $r$ be the smallest index such that $a_r \neq 0$. Clearly, $x^r$ is then the smallest power of $x$ present in $S(x)$.
    Thus, $S(x) = x^rS_0(x)$, where $S_0(x)$ is some non-zero polynomial. Note that $S_0(0) \neq 0$ as that would contradict $a_r \neq 0$.
    We now observe the following by setting $t = \frac{P_0(0)}{Q_0(0)}$:
    \begin{multline*}
        \alpha = \lim_{x \to 0^+} \frac{\frac{P(x)}{Q(x)} - t}{x^r} =
        \lim_{x \to 0^+} \frac{\frac{x^p P_0(x)}{x^q Q_0(x)} - \frac{P_0(0)}{Q_0(0)}}{x^r} =
        \lim_{x \to 0^+} \frac{\frac{P_0(x)Q_0(0) - P_0(0)Q_0(x)}{Q_0(x)Q_0(0)}}{x^r} = \\
        \lim_{x \to 0^+} \frac{\frac{x^rS_0(x)}{Q_0(x)Q_0(0)}}{x^r} =
        \lim_{x \to 0^+} \frac{S_0(x)}{Q_0(x)Q_0(0)} =
        \frac{S_0(0)}{Q_0(0)Q_0(0)} \neq
        0
    \end{multline*}
    Thus concluding the current case and proof.
\end{proof}

A series of functions $f_1(x), \dots, f_n(x)$ is said to be \textit{linearly dependent} over $\mathbb{Q}$
if there are some $c_1, \dots, c_n \in \mathbb{Q}$ not all zero such that
$c_1 f_1(x) + \dots + c_nf_n(x) = 0$ for all $x$ over the common domain of $f_1(x), \dots, f_n(x)$.
If no such $c_1, \dots, c_n$ exist, then the functions are linearly independent.
\begin{claim}\label{claim_expression_does_not_reduce_further}
    Let $B_1, \dots, B_n \in \mathbb{N}$ be distinct and positive.
    Then, the functions $1, R(x^{B_1}), \dots, R(x^{B_n})$
    are linearly independent over $\mathbb{Q}$.
\end{claim}
\begin{proof}
    Suppose for contradiction that $1, \frac{P(x^{B_1})}{Q(x^{B_1})}, \dots, \frac{P(x^{B_n})}{Q(x^{B_n})}$
    are linearly dependent over $\mathbb{Q}$. One may then find $c_1, \dots, c_n, d \in \mathbb{Q}$
    not all zero such that, for all positive $x \in \mathbb{Q}$,
    \begin{equation*}
        \sum_{j=1}^n c_j \cdot \frac{P(x^{B_j})}{Q(x^{B_j})} + d = 0.
    \end{equation*}
    Now, fix $t, r, \alpha$ as the values promised by Claim~\ref{claim_limit_for_lin_indep} and write
    $T(x) = \frac{R(x) - t}{x^r}$. For all $1 \leq j \leq n$ we then have that
    $R(x^{B_j}) = x^{rB_j} T(x^{B_j}) + t$. Thus, the linear dependence identity can be written as:
    \begin{equation*}
        \sum_{j=1}^n c_j \Big( x^{rB_j} T(x^{B_j}) + t \Big) + d = 
        \sum_{j=1}^n c_j x^{rB_j} T(x^{B_j}) + D =
        0,
    \end{equation*}
    where $D = d + t\sum_{j=1}^n c_j$.
    
    Discarding terms with $c_j = 0$ suppose, first, that $D = 0$. Now, let $j$ be such that $rB_j = \min\{ rB_\ell \mid 1 \leq \ell \leq n, c_\ell \neq 0 \}$.
    Without loss of generality, we assume that $j=1$. Dividing the above by $x^{rB_j}$ and taking the limit
    of the left-hand side we obtain the following:
    \begin{equation*}
        \lim_{x \to 0^+} c_1 T(x^{B_1}) + \sum_{j=2}^n c_j x^{rB_j - rB_1} T(x^{B_j}) =
        c_1 \alpha + \sum_{j=2}^n c_j \cdot 0 \cdot \alpha = c_1 \alpha,
    \end{equation*}
    as for all $j \geq 2$, we have that $rB_j - rB_1 > 0$.
    Notice that since the original expression is zero, we have that the limit above is also zero.
    But neither $c_1 = 0$, nor $\alpha = 0$ thus resulting in a contradiction.

    Now, suppose that $D \neq 0$. But then, if all $rB_\ell$ with $c_\ell \neq 0$ are positive
    we have
    \begin{equation*}
        \lim_{x \to 0^+} \sum_{j=1}^n c_j x^{rB_j} T(x^{B_j}) + D =
        \sum_{j=1}^n c_j \cdot 0 \cdot \alpha + D = D.
    \end{equation*}
    But $D \neq 0$ thus resulting in a contradiction.
    Otherwise, if some $rB_\ell$ with $c_\ell \neq 0$ is negative, again take
    $rB_j = \min\{ rB_\ell \mid 1 \leq \ell \leq n, c_\ell \neq 0 \}$.
    Assuming w.l.o.g. that $j = 1$ we again have that
    \begin{equation*}
        \lim_{x \to 0^+} c_1 T(x^{B_1}) + \sum_{j=2}^n c_j x^{rB_j - rB_1} T(x^{B_j}) + x^{-rB_1}D =
        c_1\alpha + \sum_{j=1}^n c_j \cdot 0 \cdot \alpha + 0D = c_1\alpha.
    \end{equation*}
    But again $c_1\neq 0$ and $\alpha \neq 0$ thus contradicting our initial assumption and concluding the proof.
\end{proof}

Finally, writing $\widehat{Z}_{i, k} = \bigcup_{r \geq 1}\widehat{Z}_{i, k, r}$ we show:
\begin{claim}\label{claim_alg_indep_to_lin_appendix}
    Suppose that each finite subset of $\bigcup_{j=-1}^i\widehat{Z}_{j, k}$ is algebraically independent over $\mathbb{Q}$.
    Then each finite subset of $\{ 1\} \cup \bigcup_{j=-1}^i Z_{j, k}$ is linearly independent over $\mathbb{Q}$.
\end{claim}
\begin{proof}
    Suppose for contradiction that $\{ 1 \} \cup \bigcup_{j=-1}^i Z_{j, k}$ is not linearly independent over $\mathbb{Q}$.
    Then, there are some elements $z_1, \dots z_m \in \bigcup_{j=-1}^i Z_{j, k}$
    and $c_1, \dots, c_m, d \in \mathbb{Q}$ not all zero
    such that $\sum_{j=1}^m c_j z_j + d = 0$.
    Recall that $z_j = R(s_1^{a_{j, 1}} \cdot {\dots} \cdot s_n^{a_{j, n}})$
    with $a_{j, 1}, \dots, a_{j, n} \in \mathbb{N}$ not all zero and $s_1, \dots, s_n \in \bigcup_{j=-1}^i \widehat{Z}_{j, k}$.
    We can rewrite the relation above as 
    \begin{equation*}
        \sum_{j=1}^m c_j \cdot \frac{P(s_1^{a_{j, 1}} \cdot {\dots} \cdot s_n^{a_{j, n}})}{Q(s_1^{a_{j, 1}} \cdot {\dots} \cdot s_n^{a_{j, n}})}
            + d = 0.
    \end{equation*}
    Since $Q$ is non zero on values above $0$ we multiply by the denominators and obtain:
    \begin{equation*}
        \sum_{j=1}^m c_j \cdot P(s_1^{a_{j, 1}} \cdot {\dots} \cdot s_n^{a_{j, n}}) \cdot
        \prod_{\ell = 1}^{\ell \neq j} Q(s_1^{a_{\ell, 1}} \cdot {\dots} \cdot s_n^{a_{\ell, n}})
            + d \prod_{j = 1}^m Q(s_1^{a_{j, 1}} \cdot {\dots} \cdot s_n^{a_{j, n}}) = 0.
    \end{equation*}
    Clearly, when symbols $s_1, \dots, s_n$ are replaced by variables, the expression above is a polynomial.
    Assuming it is non-zero, we have that $\bigcup_{j=-1}^i\widehat{Z}_{j, k}$ is not algebraically 
    independent over $\mathbb{Q}$ thus contradicting our initial assumption.
    Thus the only thing left to show is that the above is non-zero.

    Clearly, the polynomial mentioned above is non-zero if the following expression is non-zero:
    \begin{equation*}
        \sum_{j=1}^m c_j \cdot \frac{P(y_1^{a_{j, 1}} \cdot {\dots} \cdot y_n^{a_{j, n}})}{Q(y_1^{a_{j, 1}} \cdot {\dots} \cdot y_n^{a_{j, n}})}
            + d.
    \end{equation*}
    Concentrating on the above
    let $\beta = \max\{ a_{j, \ell} \mid 1 \leq j \leq m, 1 \leq \ell, \leq n \} + 1$
    and substitute each $y_\ell$ with $x^{\beta^\ell}$.
    Now, write $B_\ell = \beta^1 a_{\ell, 1} + \dots + \beta^{n}a_{\ell, n}$.
    Here $B_\ell$ is the $\beta$-ary encoding of the tuple $(a_{\ell, 1}, \dots, a_{\ell, n})$.
    Notice that each tuple is distinct as $z \in Z_{j, k}$ is formed by applying $R$
    to a non empty product of (possibly repeated) elements in $\widehat{Z}_{j, k}$ multiplied
    by some (possibly empty) product of (possibly repeated) elements in $\{ e^{\frac{1}{\kappa}}\} = \widehat{Z}_{-1, k}$.
    Since each tuple is distinct, we have that all $B_\ell$ are as well.
    What is obtained is the following expression in a single variable:
    \begin{equation*}
        \sum_{j=1}^m c_j \cdot \frac{P(x^{B_j})}{Q(x^{B_j})}
            + d.
    \end{equation*}
    But then, by Claim~\ref{claim_expression_does_not_reduce_further}, the functions 
    $1, \frac{P(x^{B_1})}{Q(x^{B_1})}, \dots, \frac{P(x^{B_m})}{Q(x^{B_m})}$
    are linearly independent. Thus, there is some value in $\mathbb{Q}$
    under which the expression above is defined and non-zero as required.
\end{proof}

Once again, the only missing part of our proof is that of algebraic independence of (finite subsets of) $\widehat{Z}_{i, k}$ over $\mathbb{Q}$.
This is again achieved by Schanuel's conjecture.
\begin{lemma}\label{lemma_algeb_indep_full_case}
    Under Schanuel's conjecture,
    finite subsets of $\bigcup_{j=-1}^L \widehat{Z}_{j, k}$ are algebraically independent over $\mathbb{Q}$.
\end{lemma}
\begin{proof}
    The proof is in similar spirit to the one in the main body of the work.
    Let us return to the setting of sets $\widehat{Z}_{i, k, r}$ 
    and proceed by induction on $-1 \leq i \leq L$ maintaining
    \begin{enumerate}
        \item $\displaystyle \bigcup_{j=-1}^{i} \widehat{Z}_{j, k, r}$ is algebraically independent over $\mathbb{Q}$,
        \item the sets $\widehat{Z}_{-1, k, r}, \dots, \widehat{Z}_{i, k, r}$ are pairwise disjoint; and so are $\{1\}, Z_{-1, k, r}, \dots, Z_{i, k, r}$,
        \item each entry of $Z_{-1, k, r}, \dots, Z_{i, k, r}$
        is positive.
    \end{enumerate}
    Again the required result will be achieved by proving the above for each $r \geq 1$, as
    $\widehat{Z}_{i, k}$ is the union of $\widehat{Z}_{i, k, r}$ for all $r \geq 1$ and $\widehat{Z}_{i, k, r} \subseteq \widehat{Z}_{i, k, r+1}$.

    For the base case note that $\bigcup_{j=-1}^{-1 } \widehat{Z}_{j, k, r} = \{e^{\frac{1}{\kappa}}\}$, where
    $\kappa$ is a non-zero integer value that is the least common denominator of the positive rationals $l_1, \dots, l_k$ as well as the weights of $\gnn$.
    Hence, condition 1 holds trivially as, by the Lindemann–Weierstrass Theorem, $e^{\frac{1}{\kappa}}$
    is algebraically independent over $\mathbb{Q}$.
    Condition 2 is immediate in regards to $\widehat{Z}_{-1, k, r}$.
    Suppose that it does not hold in regards to $\{1\}, Z_{-1, k, r}$.
    Then, $1 \in Z_{-1, k, r}$ meaning that there is some positive rational $l_j$
    for which $\sigma(l_j) = 1$. But then $P(e^{l_j}) = Q(e^{l_j})$
    and thus $P(e^{l_j}) - Q(e^{l_j}) = 0.$ But since $e^{l_j}$ is transcendental
    we have that $P(x) - Q(x) = 0$ for all $x$. Thus, $\frac{P(x)}{Q(x)} = 1$ for all
    $x$ contradicting our requirement that the function is strictly increasing for $x > 1$. Condition 3 follows from the fact
    that $l_1, \dots, l_k$ are positive rationals
    thus making $e^{\frac{z}{\kappa}} > 1$; it being understood that
    $z$ is any sum (with possible repeats) of $l_1, \dots, l_k$. By definition,
    $R(x)$ is positive on $x>1$, thus $R(e^z)$ is positive
    as required.

    Now, suppose the above conditions hold for $-1 \leq i < L$.
    We show that the same is true for $i+1$ in place of $i$.
    By condition 1 of our inductive hypothesis, $\bigcup_{j=-1}^i \widehat{Z}_{j, k, r}$ is algebraically independent over $\mathbb{Q}$.
    Hence, by invoking Claim~\ref{claim_alg_indep_to_lin_appendix} on appropriate subsets, the set
    $\{ 1 \} \cup \bigcup_{j=-1}^{i}  Z_{j, k, r}$
    is linearly independent over $\mathbb{Q}$, and thus so is 
    $S = \{ \frac{1}{\kappa} \} \cup \bigcup_{j=-1}^{i} \{ \frac{z}{\kappa} \mid z \in Z_{j, k, r} \}$
    as $\kappa$ is a positive integer.
    With that, we may apply Schanuel's conjecture on $S$ and, in doing so, obtain
    $F=\mathbb{Q}(S \cup \bigcup_{j=-1}^{i+1} \widehat{Z}_{j, k, r})$ such that
    $\mathsf{trdeg}_{\mathbb{Q}}(F) \geq |S|$.
    When turning to analyse the size of $S$.
    Notice, first, that $|Z_{i, k, r}| = |\widehat{Z}_{i+1, k, r}|$.
    Now, by condition 2 of our induction hypothesis, $\{1\}, Z_{-1, k, r}, \dots, Z_{i, k, r}$
    are all pairwise disjoint. Thus,
    \begin{equation*}
        |S| = 
        |\{ 1 \} \cup \bigcup_{j=-1}^{i} Z_{j, k, r}| = 
        1 + \sum_{j=-1}^i |Z_{j, k, r}| =
        |\widehat{Z}_{-1, k, r}| + \sum_{j=0}^{i+1} |\widehat{Z}_{j, k, r}| = \sum_{j=-1}^{i+1} |\widehat{Z}_{j, k, r}|.
    \end{equation*}
    Now, write $E = \mathbb{Q}(\bigcup_{j=-1}^{i+1} \widehat{Z}_{j, k, r})$.
    Notice that $S \subseteq E$ as elements of $S$ are either rational,
    already in the set $\bigcup_{j=-1}^{i+1} \widehat{Z}_{j, k, r}$,
    or take the form $\frac{1}{\kappa} \cdot R(z) = \frac{1}{\kappa} \cdot  \frac{P(z)}{Q(z)}$, where $z$
    is a  monomial over $\bigcup_{j=-1}^{i} \widehat{Z}_{j, k, r}$ and $\kappa$
    is a positive rational.
    By condition 3 of our induction hypothesis, $z > 1$ as it is the result
    of exponentiation by a positive power.
    Thus,
    by definition of $R(x)$, we have $Q(z) \neq 0$.
    Since fields are closed under multiplicative inverses,
    we obtain $E = F$
    and thus $\mathsf{trdeg}_{\mathbb{Q}}(E) \geq |S| = \sum_{j=-1}^{i+1} |\widehat{Z}_{j, k, r}|$.
    The above combined with Lemma~\ref{lemma_almost_trivial_alg_indep} gives us
    that $\bigcup_{j=-1}^{i+1} \widehat{Z}_{j, k, r}$ is algebraically independent over $\mathbb{Q}$.
    This is then exactly condition 1.

    Condition 2 for $\widehat{Z}_{-1, k, r}, \dots, \widehat{Z}_{i+1, k, r}$
    is achieved noting that, by I.H., $\{1\}, Z_{-1, k, r}, \dots, Z_{i, k, r}$
    are pairwise disjoint sets and exponentiation (of $e^{\frac{1}{\kappa}}$) is injective.
    Now, suppose that $Z_{i+1, k, r} \cap Z_{j, k, r} \neq \emptyset$ for some $0 \leq j \leq i$.
    Let the common element be $z$. Clearly, $z$ equals $R(s_1) = R(s_2)$, where $s_1$
    is a
    monomial over $\widehat{Z}_{i+1, k, r}$ multiplied with some non-negative power of $e^{\frac{1}{{\kappa}}}$ and $s_2$
    a monomial over $\widehat{Z}_{j, k, r}$ multiplied by some non-negative power of $e^{\frac{1}{{\kappa}}}$.
    But, by our argument above, $\widehat{Z}_{i+1, k, r} \cap \widehat{Z}_{j, k, r} = \emptyset$.
    Additionally, by condition 1 and 2 (already established above for $\widehat{Z}_{i+1, k, r}$),
    the sets $\widehat{Z}_{i+1, k, r}, \widehat{Z}_{j, k, r}$, $\{e^{\frac{1}{{\kappa}}}\}$
    are distinct and their union is algebraically independent over $\mathbb{Q}$.
    Hence, $s_1 \neq s_2$. Notice that
    $s_1, s_2 > 1$ as a consequence of condition 3.
    Indeed, by I.H., all entries of $Z_{-1, k, r}, \dots, Z_{i, k, r}$
    are positive. By definition, $s_1$ and $s_2$ are then $e$ raised to a positive power.
    Then, since $R(x)$ is strictly increasing on $x > 1$, we have that
    $R(s_1) \neq R(s_2)$ thus contradicting our initial assumption.
    The case where $Z_{i+1, k, r} \cap Z_{-1, k, r} \neq \emptyset$ is similar
    as $\widehat{Z}_{i+1, k, r}, \{e^{\frac{1}{{\kappa}}}\}$ are distinct and, when taken together, 
    algebraically independent over $\mathbb{Q}$. Finally,
    $Z_{i+1, k, r} \cap \{1\} = \emptyset$ follows from the fact that
    $1$ is algebraic over $\mathbb{Q}$ whilst elements of $Z_{i+1, k, r}$ are not.

    Condition 3 holds following the same observations as for the base case. Indeed, by I.H. entries of $Z_{i, k, r}$ are all positive.
    Thus,
    $e^{\frac{z_1}{\kappa}} \cdot {\dots} \cdot e^{\frac{z_n}{\kappa}} > 1$, for
    any choice of $z_1, \dots, z_n \in Z_{i, k, r} \cup \{ 1 \}$.
    Hence, again by definition, $R(e^{\frac{z_1}{\kappa}} \cdot {\dots} \cdot e^{\frac{z_n}{\kappa}}) > 0$.
\end{proof}

\begin{claim}
    Under Schanuel's conjecture, the vertex embeddings produced by $\gnn$ 
    partition the class of vertices into the same equivalence classes as the $L$-round
    1-WL test. Additionally, the graph embeddings produced by $\gnn$
    partition the class of graphs into the same equivalence classes as
    the $L$-round 1-WL test.
\end{claim}
\begin{proof}
    For the first statement we show
    that each of the functions $\widetilde{\lambda}^{(0)}, \dots, \widetilde{\lambda}^{(L)}$ are effectively injective on their respective domains and invoke Lemma~\ref{lemma_prelim_lin_indep_implies_wl}.
    For the second statement, we show that each finite subset
    of vertex embeddings produced by $\gnn$
    is linearly independent over $\mathbb{Q}$.

    Let us first take a look at $\widetilde{\lambda}^{(0)}$.
    Take any featured graph $G = (V, E, \lambda)$ and any $v, u \in V$.
    Suppose for contradiction that $\lambda^{(0)}(v) = \lambda^{(0)}(u)$,
    but the initial features are such that $\lambda(v) \neq \lambda(u)$.
    Inspecting the initialisation function we see that its first layer
    gives vectors $\alpha = \sigma(\lambda(v) + \pmb{1}_N)$ and
    $\beta = \sigma(\lambda(u) + \pmb{1}_N)$.
    Since $\lambda(v) + \pmb{1}_N$ and $\lambda(u) + \pmb{1}_N$
    contain only positive entries we have that
    $\alpha \neq \beta$ as $R(y)$ is strictly increasing on $y > 1$, thus 
    ensuring that $\sigma(x)$ strictly increases on $x > 0$.
    It is in the second layer of the initialisation that 
    we must have $\sigma\big((2^1, \dots, 2^{N}) \times \alpha\big) = 
    \sigma\big((2^1, \dots, 2^{N}) \times \beta\big)$.
    Suppose first that $(2^1, \dots, 2^{N}) \times \alpha = (2^1, \dots, 2^{N}) \times \beta$.
    By Claim~\ref{claim_appendix_computation_in_Z_sets},
    the entries of $\alpha$ and $\beta$ belong to $Z_{-1, k}$ for $k$
    large enough. By Lemma~\ref{lemma_algeb_indep_full_case} and Claim~\ref{claim_alg_indep_to_lin_appendix},
    the values of $Z_{-1, k} = \{s_1, \dots, s_k \}$
    are linearly independent over $\mathbb{Q}$.
    Thus, it must be the case that
    $\sum_{1 \leq j \leq N}^{\alpha[j] = s_\ell} 2^j = 
     \sum_{1 \leq j \leq N}^{\beta[j] = s_\ell} 2^j$ for all
     $1 \leq \ell \leq k$; it being understood that $\alpha[j]$ is the
     $j$-th entry of $\alpha$ (and similarly for $\beta$).
     But then, for a given $j$, the respective sum is a binary
     encoding of positions where $s_j$ resides in $\alpha$ and $\beta$.
     Thus, $\alpha[j] = \beta[j]$ for all $1 \leq j \leq N$ 
     contradicting $\alpha \neq \beta$.
     Lastly, writing $\alpha' = (2^1, \dots, 2^{N}) \times \alpha$
     and $\beta' = (2^1, \dots, 2^{N}) \times \beta$
     we must have that $\sigma(\alpha') = \sigma(\beta')$.
     But $\alpha', \beta' > 0$ thus resulting in $e^{\alpha}, e^{\beta'} > 1$
     and $\sigma(\alpha') \neq \sigma(\beta')$, a contradiction to our initial assumption
     arising from the fact $R(x)$ is strictly increasing for $x > 1$.

     For $\lambda^{(i)}$ with $1 \leq i \leq L$ we proceed as follows.
     Suppose for contradiction
    that $\lambda^{(i)}(v) = \lambda^{(i)}(u)$, but
    $$
    \big( \lambda^{(i-1)}(v), \msl \lambda^{(i-1)}(v') \mid v' \in N(v) \msr \big) \neq
    \big( \lambda^{(i-1)}(u), \msl \lambda^{(i-1)}(u') \mid u' \in N(u) \msr \big).$$ But then,
    \begin{equation*}
        \sigma(C_{i} \lambda^{(i-1)}(v) + A_{i} \sum_{v' \in N(v)} \lambda^{(i-1)}(v') + b_{i}) = 
    \sigma(C_{i} \lambda^{(i-1)}(u) + A_{i} \sum_{u' \in N(u)} \lambda^{(i-1)}(u') + b_{i}).
    \end{equation*}
    By Claim~\ref{claim_appendix_computation_in_Z_sets}
    the embeddings given by $\lambda^{i-1}$ are all in $Z_{i-1, k}$
    for $k$ large enough.
    It is easy to verify that the sums inside the activation function
    are positive. Since $\sigma(x)$ is strictly increasing when
    $x > 0$ we have that, in fact, 
    \begin{equation*}
        C_{i} \lambda^{(i-1)}(v) + A_{i} \sum_{v' \in N(v)} \lambda^{(i-1)}(v') + b_{i} = 
    C_{i} \lambda^{(i-1)}(u) + A_{i} \sum_{u' \in N(u)} \lambda^{(i-1)}(u') + b_{i}.
    \end{equation*}
    Thus, supposing the relevant values of $Z_{i-1, k}$
    are $s_1, \dots, s_n$,
    we obtain the following relation:
    \begin{equation*}
        C_{i} s_r + A_{i} \sum_{\ell = 1}^{n} \alpha_\ell s_\ell = 
        C_{i} s_t + A_{i} \sum_{\ell = 1}^{n} \beta_\ell s_\ell,
    \end{equation*}
    where all $\alpha_\ell, \beta_\ell$ are in $\mathbb{N}$
    and represent the number of neighbours (of $v$ and $u$ respectively) with the embedding
    $s_\ell$. It is to be understood that
    $s_r = \lambda^{(i-1)}(v)$ and $s_t = \lambda^{(i-1)}(u)$.
    Supposing for the moment $s_r = s_t$ we have that
    $\sum_{\ell = 1}^{n} \alpha_\ell s_\ell = \sum_{\ell = 1}^{n} \beta_\ell s_\ell$.
    Since values of $Z_{i-1, k}$ are linearly independent over $\mathbb{Q}$
    (Lemma~\ref{lemma_algeb_indep_full_case} with Claim~\ref{claim_alg_indep_to_lin_appendix})
    we have that $\alpha_\ell = \beta_\ell$ for all $1 \leq \ell \leq n$.
    Thus, there is a bijection $f : N(v) \to N(u)$
    such that $\lambda^{(i-1)}(v') = \lambda^{(i-1)}(f(v'))$.
    Thus,
    $\big( \lambda^{(i-1)}(v), \msl \lambda^{(i-1)}(v') \mid v' \in N(v) \msr \big)
    =
    \big( \lambda^{(i-1)}(u), \msl \lambda^{(i-1)}(f(v')) \mid v' \in N(v) \msr \big)
    =
    \big( \lambda^{(i-1)}(u), \msl \lambda^{(i-1)}(u') \mid u' \in N(u) \msr \big)$
    thus contradicting our initial assumption.
    
    Now suppose that $s_r \neq s_t$. If $\alpha_\ell = \beta_\ell$ for
    all $1 \leq \ell \leq n$, then $C_{i}s_r = C_{i}s_t$.
    Since $C_{i} > 0$ this results in a contradiction.
    Now, if $\alpha_\ell \neq \beta_\ell$ for some choice of $\ell$, we obtain the following by grouping dependent terms together:
    \begin{equation*}
        C_{i} s_r + A_{i} \alpha_r s_r = A_{i} \beta_r s_r \quad \text{ and } \quad
        A_{i} \alpha_t s_t = C_{i} s_t + A_{i} \beta_t s_t.
    \end{equation*}
    Focusing on the first equation, recall that $\frac{C_{i}}{A_{i}} \not\in \mathbb{N}$.
    In fact, since $A_{i}, C_{i} > 0$, we have that 
    $\frac{C_{i}}{A_{i}} \not\in \mathbb{Z}$.
    But then $\beta_r - \alpha_r = \frac{C_{i}}{A_{i}}$.
    Since $\alpha_r, \beta_r \in \mathbb{N}$ we have that
    $\frac{C_{i}}{A_{i}} \in \mathbb{Z}$, a contradiction.

    For the second statement of the claim we proceed as follows.
    Take any finite subset $\{s_1, \dots, s_n\}$ such that
    $s_j = \lambda^{(L)}(v_j)$ for some pointed graph $(G_j, v_j)$.
    Now, let $G$ be the disjoint union of $G_1, \dots, G_n$.
    Claim~\ref{claim_appendix_computation_in_Z_sets} then establishes
    $s_1, \dots, s_n \in Z_{L, k}$ for $k$ large enough.
    Lemma~\ref{lemma_algeb_indep_full_case} and Claim~\ref{claim_alg_indep_to_lin_appendix}
    then establish that these values are linearly independent over $\mathbb{Q}$.
\end{proof}

\section{Proof of Proposition~\ref{prop_1_WL_over_virtual}}
\propWLovervirtual*
\begin{proof}
The proof is comprised of two inductive arguments.
First, we show that $\mu^{(i+1)}(v) = \mu^{(i+1)}(u)$
implies $\mu^{(i)}(v) = \mu^{(i)}(u)$.
For the base case suppose that 
$\mu^{(1)}(v) = \mu^{(1)}(u)$, but
$\mu^{(0)}(v) \neq \mu^{(0)}(u)$.
If $v$ and $u$ are ordinary vertices, then 
$\mu^{(1)}(v) = \mathsf{hash}(\mu^{(0)}(v), \dots) \neq 
\mathsf{hash}(\mu^{(0)}(u), \dots) =
\mu^{(1)}(u)$, a contradiction.
If $v$ and $u$ are virtual then, by definition, 
$\mu^{(0)}(v) = \mathsf{hash}(\mathsf{virtual}) = \mu^{(0)}(u)$.
The case where one node is virtual and the other is ordinary is immediate as then $\mu^{(j)}(v) \neq \mu^{(j)}(u)$ for any $j \geq 0$.

Let us assume the required condition
holds for $i > 0$. We extend this to the $i+1$ setting as follows.
Suppose $\mu^{(i+1)}(v) = \mu^{(i+1)}(u)$.
If $v$ and $u$ are ordinary vertices, then 
$\mu^{(i)}(v) \neq \mu^{(i)}(u)$ would result in 
a similar contradiction as in the base case.
On the other hand suppose $v$ and $u$ are virtual and
that $\mu^{(i+1)}(v) = \mu^{(i+1)}(u)$,
but $\mu^{(i)}(v) \neq \mu^{(i)}(u)$.
Writing $S = \msl \mu^{(i-1)}(v') \mid v' \in N(v) \msr$
and $Z  = \msl \mu^{(i-1)}(u') \mid u' \in N(u) \msr$
we observe that $S \neq Z$.
But then, $S^+ = Z^+$, where
$S^+ = \msl \mu^{(i)}(v') \mid v' \in N(v) \msr$
and $Z^+  = \msl \mu^{(i)}(u') \mid u' \in N(u) \msr$.
Now, let $f$ be a bijection
between $N(v)$ and $N(u)$ such that $\mu^{(i)}(v') = \mu^{(i)}(f(v'))$
for all $v' \in N(v)$. But then, by I.H., the same bijection
gives $\mu^{(i-1)}(v') = \mu^{(i-1)}(f(v'))$
and establishes $S = Z$, a contradiction.
The case where one node is virtual and the other is ordinary is immediate.

We now prove that
$\mu^{(L)}(v) = \mu^{(L)}(u)$ implies
$\lambda^{(L)}(v) = \lambda^{(L)}(u)$
by induction on $0 \leq i \leq L$.
The base case $i=0$ is immediate as $\lambda(v) = \lambda(u)$
if and only if $\mu^{(0)}(v) = \mu^{(0)}(u)$.
Since $\lambda^{(0)}$ is an FNN for ordinary vertices and produced the same vector for virtual ones, we have that
$\mu^{(0)}(v) = \mu^{(0)}(u)$ implies
$\lambda^{(0)}(v) = \lambda^{(0)}(u)$.

For the inductive step assume that our hypothesis
holds for $0 \leq i < L$.
We show that it holds for $i+1$ as well.
For this purpose suppose $\mu^{(i+1)}(v) = \mu^{(i+1)}(u)$.
Then, as argued above, $\mu^{(i)}(v) = \mu^{(i)}(u)$ and
$\msl \mu^{(i)}(v') \mid v' \in N(v) \msr = 
\msl \mu^{(i)}(u') \mid u' \in N(u) \msr$.
But then, by our inductive hypothesis,
$\lambda^{(i)}(v) = \lambda^{(i)}(u)$
as well as
$\msl \lambda^{(i)}(v') \mid v' \in N(v) \msr = 
\msl \lambda^{(i)}(u') \mid u' \in N(u) \msr$.
Recall that, before activation, $\lambda^{(i+1)}(v)$
is a linear combination of $\sum \msl \lambda^{(i)}(v') \mid v' \in N(v) \msr$ and $\lambda^{(i)}(v)$; and, in fact, so is $\lambda^{(i+1)}(u)$.
It is then easy to see that $\lambda^{(i+1)}(v) = \lambda^{(i+1)}(u)$ as required.

To see that $\mu_G = \mu_H$ implies $\lambda_G = \lambda_H$ one needs only
note that $\msl \mu^{(L)}(v) \mid v \in V \msr = 
\msl \mu^{(L)}(u) \mid u \in U \msr$; it being understood that
$V$ is the set of vertices of $G$ and $U$ is the set of vertices of $H$.
Then, $\sum \msl \lambda^{(L)}(v) \mid v \in V \msr = 
\sum \msl \lambda^{(L)}(u) \mid u \in U \msr$
as required.
\end{proof}

\section{Full proof of Theorem~\ref{prop:converse}}
\propConverse*
In this section we only fill in the parts of the above theorem.
Namely we show that $G_Q$ and $G_R$ as constructed in the theorem
are distinguishable by $L$-round 1-WL, that  $\gnn$ is such that
$\lambda_{G_Q} = \lambda_{G_R}$, and show that
the above theorem holds with $\sigma \in \{ \sinh, \cosh\}$.

We proceed by showing the first claim.
For this purpose recall the pointed graphs
$(G, v_1), \dots, (G, v_n)$,
values $z_i = \exp (A_L'\lambda^{(L-1)}(v_i))$, the identity
\begin{equation*}
        \sum_{i=1}^{q}
        \alpha_i
        z_1^{\beta_{i,1}}\cdots z_n^{\beta_{i,n}}
        \exp(b'_L)
        +
        \alpha_{q+1}\exp(b'_L)
        =
        \sum_{j=1}^{r}
        \gamma_j
        z_1^{\delta_{j,1}}\cdots z_n^{\delta_{j,n}}
        \exp(b'_L),
\end{equation*}
and graph construction:
\begin{enumerate}
        \item for $1\leq i\leq q$, add $\alpha_i$ disjoint copies
        of the graph constructed as follows:
        \begin{enumerate}
            \item for $1\leq k\leq n$, add $\beta_{i,k}$ fresh
            copies of the pointed graph $(G,v_k)$;
            \item create a fresh virtual node $u$ and add an edge
            from $u$ to the distinguished vertex in each copy
            added in the previous step;
        \end{enumerate}
        \item add $\alpha_{q+1}$ fresh isolated virtual vertices;
        \item for $1\leq j\leq r$ and $1\leq k\leq n$, add
        $\gamma_j\delta_{j,k}$ disjoint copies of $(G,v_k)$.
    \end{enumerate}
The graph $G_R$ is defined symmetrically: interchange
    $q$ and $r$ as well as $(\alpha,\beta)$ and $(\gamma,\delta)$,
    and omit Step~2.
\begin{claim}\label{appendix_claim_about_QR_and_WL}
    For $G_Q$ and $G_R$ as in Theorem~\ref{prop:converse}
    we have $\mu_{G_Q} \neq \mu_{G_R}$.
\end{claim}
\begin{proof}
    Recall that $G_Q$ and $G_R$ are built from pointed graphs 
    $(G, v_1), \dots, (G, v_n)$ that are pairwise distinguishable
    by the $(L{-}1)$-round 1-WL algorithm.
    Looking at steps 1 (a) and 3 of the 
    construction in Theorem~\ref{prop:converse}
    we see that there were
    $\sum_{i=1}^q \alpha_i \beta_{i, k} +
    \sum_{j=1}^r \gamma_j \delta_{j, k}$
    disjoint copies of $(G, v_k)$ added to both $G_Q$ and $G_R$.
    Notice that the only edges added were ones that originate
    from a fresh vertex introduced in step 1 (b). Thus,
    in both $G_Q$ and $G_R$, the copies of $(G, v_k)$ are coloured identically
    by the $L$-round 1-WL and, as it were, cancel each-other out
    as far as 1-WL graph distinguishability is concerned.

    We now inspect the virtual vertices added in steps 1 (b) and 2.
    Firstly, notice that vertices added in 2 have a unique
    colour when compared to those added in 1 (b).
    Indeed these vertices have no edges associated with them
    and thus, if $\alpha_{q+1} \neq 0$, it is immediate that
    $\mu_{G_Q} \neq \mu_{G_R}$.

    On the other hand suppose that $\alpha_{q+1} = 0$.
    We claim that the vertices added in step 1 (b)
    make the two graphs distinguishable.
    For this purpose let $c_1, \dots, c_n$ be
    the $(L{-}1)$-round 1-WL colours produced on the pointed vertices
    $(G, v_1), \dots, (G, v_n)$.
    Recall that they are all distinct.
    Taking any $v \in V_Q$ added in step 1 (b) we see that it takes the colour
    $\mu^{(L)}(v) = \{ (c_1 \mapsto \beta_{i, 1}), \dots, (c_n \mapsto \beta_{i, n}) \}$
    for some $1 \leq i \leq q$.
    By inspecting $Q(x_1, \dots, x_n) = R(x_1, \dots, x_n)$
    in Theorem~\ref{prop:converse} we see that
    there is no $1 \leq j \leq r$
    for which
    $(\beta_{i, 1}, \dots \beta_{i, n}) = (\delta_{j, 1}, \dots \delta_{j, n})$. Thus, by the construction in step 1, there is no $u \in V_R$
    such that $\mu^{(L)}(u) = \mu^{(L)}(v)$.
    With that we conclude $\mu_{G_Q} \neq \mu_{G_R}$.
\end{proof}

\begin{claim}\label{appendix_claim_about_QR_and_GNN}
    For $G_Q$ and $G_R$ as in Theorem~\ref{prop:converse}
    we have $\lambda_{G_Q} = \lambda_{G_R}$.
\end{claim}
\begin{proof}
    For the final global readout notice that, by construction of $G_Q$ and
    $G_R$,  
    there is a bijection
     between the vertices of $G_Q$ and $G_R$ added in  Steps~1(a) and~3 that preserves the colours attributed by the $L$-round 1-WL algorithm.
    Indeed, Step~1 adds $\alpha_i \cdot \beta_{i, k}$ copies of $(G_k,
    v_k)$ to $G_Q$ while Step~3
    does the same for $G_R$, and vice versa for $\gamma_j, \delta_{j, k}$.
    Thus, when concerned with the final readout, we need only look at the contribution of virtual nodes added
    at Steps~1(b) and~2.
    But notice that the final readout on the graphs is as follows
    \begin{equation*}
        \sum_{v \in V_Q} \lambda^{(L)}(v) = K + \exp(b'_L) \cdot Q(z_1, \dots, z_n) = K + \exp(b'_L) \cdot R(z_1, \dots, z_n) = \sum_{v \in V_R} \lambda^{(L)}(v),
    \end{equation*}
    where $K$ is the contribution of vertices that were added in Steps~1(a) and~3.
    Since the readout over the two graphs is the same, we have that $\lambda_{G_{Q}} = \lambda_{G_{R}}$, as desired.
\end{proof}
A high level explanation of how the two claims above complete
Theorem~\ref{prop:converse} is in order.
Recall that we assumed (for contradiction) that the values
$z_1, \dots, z_n$ are algebraically dependent over $\mathbb{Q}$.
This allowed us to construct two distinct polynomials $Q$ and $R$ with positive integer
coefficients such that $Q(z_1, \dots, z_n) = R(z_1, \dots, z_n)$.
They then gave rise to two graphs $G_Q$ and $G_R$
satisfying $\mu_{G_Q} \neq \mu_{G_R}$ (Claim~\ref{appendix_claim_about_QR_and_WL})
and $\lambda_Q = \lambda_R$ (Claim~\ref{appendix_claim_about_QR_and_GNN}).
But Theorem~\ref{prop:converse} requires
$\lambda_{G_1} = \lambda_{G_2} \Leftrightarrow \mu_{G_1} = \mu_{G_2}$
for all graphs $G_1, G_2$
thus resulting in a contradictions.
It must then be the case that if $\gnn$ matches the $L$-round 1-WL test
on graphs with virtual nodes, the values $z_1, \dots, z_n$
(where $z_i = \exp (A_L'\lambda^{(L-1)}(v_i))$) 
must be algebraically independent over $\mathbb{Q}$.

Lastly, we extend our reasoning to other advertised forms of activation:
\begin{claim}
    Theorem~\ref{prop:converse} extends to GNNs with
    $\sigma \in \{ \sinh, \cosh \}$ activation.
\end{claim}
\begin{proof}
    To prove the Theorem in full we must show that the values in
    $\big\{ \exp \big(A_L'\lambda^{(L-1)}(v) \big) \mid v \in V \big\}$ are algebraically independent over $\mathbb{Q}$ for any
    given graph $G = (V, E, \lambda)$.
    For the purposes of contradiction suppose this is not the case.
    Enumerating the values as $z_1, \dots, z_n$ we obtain
    a non-zero polynomial $P(x_1, \dots, x_n)$ with rational coefficients for which
    $P(z_1, \dots, z_n) = 0$.
    We again assume, by multiplying by the lowest common denominator if needed, that the coefficients in $P$ are integers.

    Let us now inspect the expression
    $R(x_1, \dots, x_n) = P(x_1, \dots, x_n) \cdot P(x_1^{-1}, \dots, x_n^{-1})$.
    Clearly, $R(z_1, \dots, z_n) = 0$.
    Now, peeking into $P(x_1, \dots, x_n)$ we see that it takes the form
    $\sum_{i=1}^m a_i M_i + b$, where each $M_i$
    is a monomial in $x_1, \dots, x_n$.
    Thus, $R(x_1, \dots, x_n)$ takes the form:
    \begin{equation*}
        \Big(\sum_{i=1}^m a_i M_i + b \Big) \cdot 
        \Big(\sum_{i=1}^m a_i M_i^{-1} + b \Big).
    \end{equation*}
    Expanding the expression above we obtain:
    \begin{equation*}
        b^2 + 
        b\sum_{i=1}^m a_i M_i + b\sum_{i=1}^m a_i M_i^{-1} +
        \sum_{i=1}^m \sum_{j=1}^m a_ia_j M_i M_j^{-1}.
    \end{equation*}
    Having the above we wish to rid ourselves of monomials that 
    that feature both negative and positive powers; i.e. $M_iM_j^{-1}$.
    For this purpose let $N$ be any monomial over $x_1, \dots, x_n$
    such that, for all $1 \leq j \leq m$,
    $NM_j^{-1} = x_1^{k_1} \cdot {\dots} \cdot x_n^{k_n}$
    for some $k_1, \dots, k_n \in \mathbb{N} \setminus \{ 0 \}$.
    Clearly $S(x_1, \dots, x_n) = NR(x_1, \dots, x_n)$ is a polynomial that takes the form:
    \begin{equation*}
        b^2N + 
        b\sum_{i=1}^m a_i N M_i + b\sum_{i=1}^m a_i N M_i^{-1} +
        \sum_{i=1}^m \sum_{j=1}^m a_ia_j N M_i M_j^{-1}.
    \end{equation*}
    Then, $S(x_1, \dots, x_n)$ can be rewritten into the compact form
    $\sum_{i=1}^{m'} a'_i N_i$, where each $a'_i \in \mathbb{Z}$ and $N_i$
    is a monomial of the form $x_1^{c_{i, 1}}, \dots, x_n^{c_{i, n}}$, where
    $c_{i, 1}, \dots, c_{i, n} \in \mathbb{N}$.
    Now, recalling that $S(x_1, \dots, x_n) = N P(x_1, \dots, x_n) \cdot P(x_1^{-1}, \dots, x_n^{-1})$, we see that $S(z_1, \dots, z_n) = 0$.
    Additionally, $S(z_1^{-1}, \dots, z_n^{-1}) = 0$.
    Thus, the relations $T_{+}(x_1, \dots, x_n)$ and $T_{-}(x_1, \dots, x_n)$ defined as:
    \begin{equation*}
        \frac{e^{b'_L} S(x_1, \dots, x_n) \pm e^{-b'_L} S(x_1^{-1}, \dots, x_n^{-1})}{2}
    \end{equation*}
    also evaluate to zero at $z_1, \dots, z_n$; it being understood that
    $\pm \equiv +$ in $T_+$ with $\pm \equiv -$ in $T_-$, 
    and with $b_L'$ being the $L$-th layer bias vector used by virtual nodes.
    Peaking inside $T_\pm(z_1, \dots, z_n)$
    we see:
    \begin{equation*}
        \sum_{i=1}^{m'} a'_i \cdot \frac{e^{b'_L} N_i(z_1, \dots, z_n) \pm e^{-b'_L} N_i(z_1^{-1}, \dots, z_n^{-1})}{2}.
    \end{equation*}
    Since $N_i$ are monomials and $z_j = e^{s_j}$ we obtain:
    \begin{equation*}
        \sum_{i=1}^{m'} a'_i \cdot \frac{(e^{s_1})^{c_{i, 1}} \cdot {\dots} \cdot
        (e^{s_n})^{c_{i, n}} \cdot e^{b'_L} \pm (e^{s_1})^{-c_{i, 1}} \cdot {\dots} \cdot
        (e^{s_n})^{-c_{i, n}} \cdot e^{-b'_L}}{2}.
    \end{equation*}
    Recalling that $\sinh(x) = \frac{e^x - e^{-x}}{2}$
    and $\cosh(x) = \frac{e^x + e^{-x}}{2}$ we obtain that
    $T_-(z_1, \dots, z_n)$ and $T_+(z_1, \dots, z_n)$
    are, respectively:
    \begin{equation*}
        \sum_{i=1}^{m'} a'_i \sinh(s_1c_{i, 1} + \dots + s_nc_{i, n} + b'_L)
        \quad \text{ and } \quad
        \sum_{i=1}^{m'} a'_i \cosh(s_1c_{i, 1} + \dots + s_nc_{i, n} + b'_L).
    \end{equation*}
    Noting that $s_1, \dots, s_n \in \{ A_L'\lambda^{(L-1)}(v)  \mid v \in V \}$,
    each expression $\sinh(s_1c_{i, 1} + \dots + s_nc_{i, n} + b'_L)$
    is then computable by a virtual node using $\sinh$ as activation.
    Indeed, taking pointed graphs $(G_1, v_1), \dots, (G_n, v_n)$
    such that $A'_L\lambda^{(L-1)}(v_j) = s_j$ we have that 
    a virtual node $u_i$ connected to $c_{i, j}$ copies of $(G_j, v_j)$
    for all $1 \leq j \leq n$
    has $\lambda^{(L)}(u_i) = \sinh(s_1c_{i, 1} + \dots + s_nc_{i, n} + b'_L)$.
    Similar observations follow for $\cosh$.
    The outermost sum can then be treated as the global readout of $\gnn$ over virtual nodes; provided that terms with negative coefficients have been moved to the right hand side of
    $T{\pm}(z_1, \dots, z_n) = 0$. Having the above one can proceed with the same graph construction as in Theorem~\ref{prop:converse}.

    In this setting one again needs to be careful so that $T_\pm(x_1, \dots, x_n)$
    does not turn out to be the zero function.
    Notice, first, that $R(x_1, \dots, x_n)$ is non-zero.
    For this let us again look at
    $P(x_1, \dots, x_n) = \sum_{i=1}^{m} a_i x_1^{d_{i, 1}} \cdot {\dots} \cdot x_n^{d_{i, n}} + b$, where $a_i \in \mathbb{Z} \setminus \{0\}$, $b \in \mathbb{Z}$ and
    $d_{i, 1}, \dots, d_{i, n} \in \mathbb{N}$ not all zero.
    Let $\beta = \max\{ d_{i, j} \mid 1 \leq i \leq m, 1 \leq j \leq n \} + 1$
    and $B_{i} = \beta^1 d_{i, 1} + \dots + \beta^n d_{i, n}$.
    Then $P(x^{\beta^1}, \dots, x^{\beta^n}) = \sum_{i=1}^m a_i x^{B_i} + b$
    is a polynomial, where $B_i$ captures via a $\beta$-ary encoding
    the sequence $(d_{i, 1}, \dots, d_{i, n})$.
    Since all degree sequences were taken to be distinct, we have that all $B_i$
    are as well. Thus the polynomial is non-zero. Without loss of generality,
    let $B_1$ and $B_m$ be, respectively, the minimum and maximum values
    amongst $B_1, \dots, B_m$. Now, inspecting $R(x^{\beta^1}, \dots, x^{\beta^n})$
    we have:
    \begin{equation*}
        b^2 + 
        b\sum_{i=1}^m a_i x^{B_i} + b\sum_{i=1}^m a_i x^{-B_i} +
        \sum_{i=1}^m \sum_{j=1}^m a_ia_j x^{B_i - B_j}.
    \end{equation*}
    Let us first suppose that $b = 0$.
    Then, multiplying by $x^{B_m - B_1}$ and taking the limit we obtain:
    \begin{equation*}
        \lim_{x \to 0} \sum_{i=1}^m \sum_{j=1}^m a_ia_j x^{B_i - B_j + B_m - B_1} = a_1 a_m.
    \end{equation*}
    Otherwise, when $b \neq 0$, we have the following by multiplying by $x^{B_m}$
    and taking the limit:
    \begin{equation*}
        \lim_{x\to 0} \Big( b^2 x^{B_m} + 
        b\sum_{i=1}^m a_i x^{B_i+B_m} + b\sum_{i=1}^m a_i x^{-B_i+B_m} +
        \sum_{i=1}^m \sum_{j=1}^m a_ia_j x^{B_i - B_j + B_m} \Big) = ba_m. 
    \end{equation*}
    Having both cases we conclude that the expression $R(x_1, \dots, x_n)$ is non-zero.
    
    Almost immediately, because of the above, $S(x_1, \dots, x_n) = N R(x_1, \dots, x_n)$
    is a non-zero polynomial. Indeed, begin by writing
    $R(x_1, \dots, x_n)$ as
    $\sum_{i=1}^{m'} a_i' x_1^{d_{i, 1}'} \cdot {\dots} \cdot x_n^{d_{i, n}'} + b'$
    where $a_i', b' \in \mathbb{Z}$, $d_{i, 1}, \dots, d_{i, n} \in \mathbb{Z}$ not all
    zero and with sequences $(d_{i, 1}, \dots, d_{i, n}) \neq (d_{j, 1}, \dots, d_{j, n})$
    for $i \neq j$. But then, multiplying by a common monomial
    $N = x_1^{d_{0, 1}} \cdot {\dots} \cdot x_n^{d_{0, n}}$ cannot result in
    terms cancelling out. Indeed, 
    $(d_{i, 1} + d_{0, 1}, \dots, d_{i, n} + d_{0, n}) = (d_{j, 1} + d_{0, 1}, \dots, d_{j, n} + d_{0, n})$ implies $(d_{i, 1}, \dots, d_{i, n}) = (d_{j, 1}, \dots, d_{j, n})$, a contradiction.
    Thus the polynomial $S(x_1, \dots, x_n)$ has at least a single term with a non-zero coefficient making the polynomial itself non-zero.

    Lastly, we need to show that $e^{b_L'} S(x_1, \dots, x_n) \pm e^{-b_L'} S(x_1^{-1}, \dots, x_n^{-1})$ is non-zero. Clearly, if the above is the case, then the relevant expressions
    $T_\pm(x_1, \dots, x_n)$ are non-zero as well.
    Inspecting the monomials (with non-zero coefficients) of $S(x_1, \dots, x_n)$ and $S(x_1^{-1}, \dots, x_n^{-1})$ we see that they are all distinct (and that
    there are no constant terms in either of the expressions).
    Since $e^{b_L'} \neq 0$ we have that
    $e^{b_L'} S(x_1, \dots, x_n) \pm e^{-b_L'} S(x_1^{-1}, \dots, x_n^{-1})$
    does not cause cancellations leaving the expression non-zero. \qedhere

\end{proof}

\section{Experiments}
%
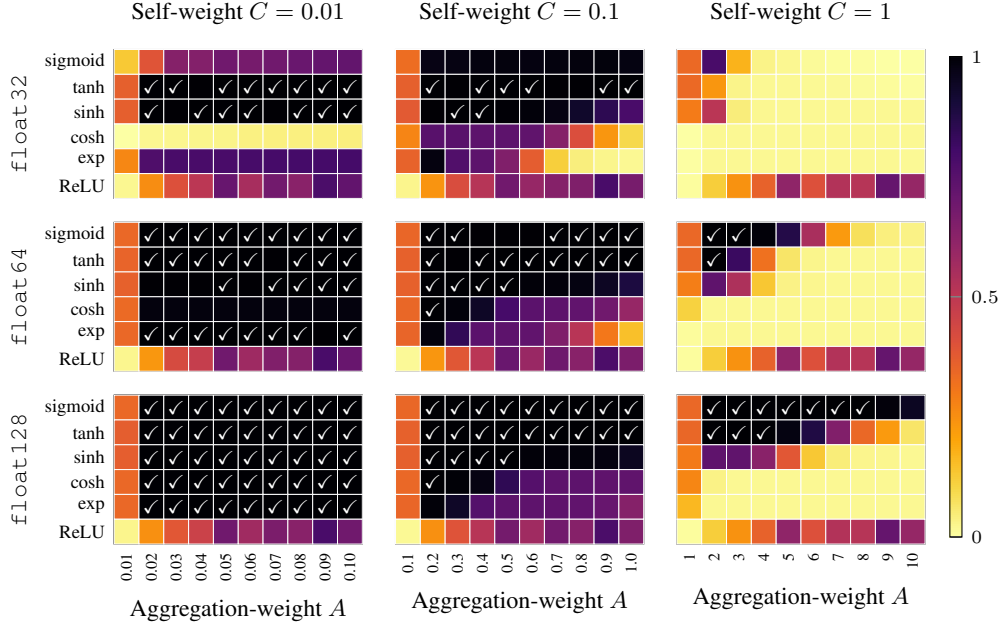
\begin{figure*}[b]
\centering
\begin{tikzpicture}
\begin{groupplot}[
  group style={
    group size=3 by 3,
    horizontal sep=1.25em,
    vertical sep=0.9em,
  },
  width=0.235\textwidth,
  height=0.141\textwidth,
  scale only axis,
  xmin=0.5, xmax=10.5,
  ymin=0.5, ymax=6.5,
  y dir=reverse,
  xtick={1,...,10},
  ytick={1,2,3,4,5,6},
  yticklabels={sigmoid,tanh,sinh,cosh,exp,ReLU},
  tick label style={font=\scriptsize},
  x tick label style={font=\tiny,rotate=90,anchor=east},
  label style={font=\small},
  title style={font=\small},
  point meta min=0,
  point meta max=1,
  colormap={inferno}{
    rgb255(0cm)=(252,255,164);
    rgb255(1cm)=(246,215,70);
    rgb255(2cm)=(252,165,10);
    rgb255(3cm)=(243,120,25);
    rgb255(4cm)=(221,81,58);
    rgb255(5cm)=(188,55,84);
    rgb255(6cm)=(147,38,103);
    rgb255(7cm)=(106,23,110);
    rgb255(8cm)=(66,10,104);
    rgb255(9cm)=(22,11,57);
    rgb255(10cm)=(0,0,4)
  },
]

\nextgroupplot[
  title={Self-weight $C=0.01$},
  ylabel={\texttt{float32}},
  xticklabels={},
]
\addplot[
  matrix plot*,
  mesh/cols=10,
  point meta=explicit,
  visualization depends on={\thisrow{z} \as \cellvalue},
  nodes near coords={%
    \pgfmathparse{equal(\cellvalue,1)}%
    \ifnum\pgfmathresult=1\relax
      $\checkmark$%
    \fi
  },
  every node near coord/.append style={
    text=white,
    font=\scriptsize,
    anchor=center,
  },
  draw=white,
  line width=0.15pt,
] table[meta=z] {
x y z
1 1 0.128
2 1 0.392
3 1 0.634
4 1 0.640
5 1 0.686
6 1 0.672
7 1 0.686
8 1 0.704
9 1 0.716
10 1 0.718
1 2 0.362
2 2 1.000
3 2 1.000
4 2 0.998
5 2 1.000
6 2 1.000
7 2 1.000
8 2 1.000
9 2 1.000
10 2 1.000
1 3 0.364
2 3 1.000
3 3 0.998
4 3 1.000
5 3 1.000
6 3 1.000
7 3 0.998
8 3 1.000
9 3 1.000
10 3 1.000
1 4 0.000
2 4 0.020
3 4 0.022
4 4 0.026
5 4 0.028
6 4 0.036
7 4 0.038
8 4 0.038
9 4 0.038
10 4 0.040
1 5 0.266
2 5 0.768
3 5 0.784
4 5 0.782
5 5 0.784
6 5 0.788
7 5 0.790
8 5 0.792
9 5 0.792
10 5 0.792
1 6 0.020
2 6 0.254
3 6 0.404
4 6 0.502
5 6 0.708
6 6 0.552
7 6 0.680
8 6 0.624
9 6 0.778
10 6 0.720
};

\nextgroupplot[
  title={Self-weight $C=0.1$},
  yticklabels={},
  xticklabels={},
]
\addplot[
  matrix plot*,
  mesh/cols=10,
  point meta=explicit,
  visualization depends on={\thisrow{z} \as \cellvalue},
  nodes near coords={%
    \pgfmathparse{equal(\cellvalue,1)}%
    \ifnum\pgfmathresult=1\relax
      $\checkmark$%
    \fi
  },
  every node near coord/.append style={
    text=white,
    font=\scriptsize,
    anchor=center,
  },
  draw=white,
  line width=0.15pt,
] table[meta=z] {
x y z
1 1 0.328
2 1 0.970
3 1 0.974
4 1 0.982
5 1 0.984
6 1 0.978
7 1 0.982
8 1 0.984
9 1 0.986
10 1 0.976
1 2 0.362
2 2 1.000
3 2 0.998
4 2 1.000
5 2 1.000
6 2 1.000
7 2 0.998
8 2 0.998
9 2 1.000
10 2 1.000
1 3 0.378
2 3 0.998
3 3 1.000
4 3 1.000
5 3 0.996
6 3 0.990
7 3 0.974
8 3 0.926
9 3 0.848
10 3 0.782
1 4 0.272
2 4 0.746
3 4 0.744
4 4 0.734
5 4 0.734
6 4 0.726
7 4 0.634
8 4 0.410
9 4 0.234
10 4 0.094
1 5 0.354
2 5 0.980
3 5 0.764
4 5 0.734
5 5 0.646
6 5 0.370
7 5 0.116
8 5 0.042
9 5 0.026
10 5 0.018
1 6 0.024
2 6 0.238
3 6 0.414
4 6 0.514
5 6 0.684
6 6 0.594
7 6 0.634
8 6 0.648
9 6 0.782
10 6 0.668
};

\nextgroupplot[
  title={Self-weight $C=1$},
  yticklabels={},
  xticklabels={},
]
\addplot[
  matrix plot*,
  mesh/cols=10,
  point meta=explicit,
  visualization depends on={\thisrow{z} \as \cellvalue},
  nodes near coords={%
    \pgfmathparse{equal(\cellvalue,1)}%
    \ifnum\pgfmathresult=1\relax
      $\checkmark$%
    \fi
  },
  every node near coord/.append style={
    text=white,
    font=\scriptsize,
    anchor=center,
  },
  draw=white,
  line width=0.15pt,
] table[meta=z] {
x y z
1 1 0.338
2 1 0.784
3 1 0.172
4 1 0.030
5 1 0.020
6 1 0.012
7 1 0.006
8 1 0.004
9 1 0.000
10 1 0.000
1 2 0.338
2 2 0.234
3 2 0.028
4 2 0.018
5 2 0.018
6 2 0.016
7 2 0.014
8 2 0.010
9 2 0.008
10 2 0.004
1 3 0.294
2 3 0.506
3 3 0.044
4 3 0.020
5 3 0.018
6 3 0.018
7 3 0.018
8 3 0.018
9 3 0.018
10 3 0.018
1 4 0.000
2 4 0.018
3 4 0.018
4 4 0.018
5 4 0.018
6 4 0.018
7 4 0.018
8 4 0.018
9 4 0.018
10 4 0.018
1 5 0.002
2 5 0.018
3 5 0.018
4 5 0.018
5 5 0.018
6 5 0.018
7 5 0.018
8 5 0.018
9 5 0.018
10 5 0.016
1 6 0.006
2 6 0.118
3 6 0.242
4 6 0.356
5 6 0.610
6 6 0.404
7 6 0.524
8 6 0.514
9 6 0.712
10 6 0.596
};

\nextgroupplot[
  ylabel={\texttt{float64}},
  xticklabels={},
]
\addplot[
  matrix plot*,
  mesh/cols=10,
  point meta=explicit,
  visualization depends on={\thisrow{z} \as \cellvalue},
  nodes near coords={%
    \pgfmathparse{equal(\cellvalue,1)}%
    \ifnum\pgfmathresult=1\relax
      $\checkmark$%
    \fi
  },
  every node near coord/.append style={
    text=white,
    font=\scriptsize,
    anchor=center,
  },
  draw=white,
  line width=0.15pt,
] table[meta=z] {
x y z
1 1 0.342
2 1 1.000
3 1 1.000
4 1 1.000
5 1 1.000
6 1 1.000
7 1 1.000
8 1 1.000
9 1 1.000
10 1 1.000
1 2 0.352
2 2 1.000
3 2 1.000
4 2 1.000
5 2 1.000
6 2 1.000
7 2 0.998
8 2 1.000
9 2 1.000
10 2 1.000
1 3 0.358
2 3 0.998
3 3 0.998
4 3 0.998
5 3 1.000
6 3 0.998
7 3 1.000
8 3 1.000
9 3 1.000
10 3 1.000
1 4 0.336
2 4 0.984
3 4 0.984
4 4 0.984
5 4 0.984
6 4 0.984
7 4 0.984
8 4 0.984
9 4 0.984
10 4 0.984
1 5 0.342
2 5 1.000
3 5 1.000
4 5 1.000
5 5 1.000
6 5 1.000
7 5 1.000
8 5 1.000
9 5 0.998
10 5 1.000
1 6 0.022
2 6 0.232
3 6 0.424
4 6 0.472
5 6 0.686
6 6 0.574
7 6 0.646
8 6 0.636
9 6 0.784
10 6 0.706
};

\nextgroupplot[
  yticklabels={},
  xticklabels={},
]
\addplot[
  matrix plot*,
  mesh/cols=10,
  point meta=explicit,
  visualization depends on={\thisrow{z} \as \cellvalue},
  nodes near coords={%
    \pgfmathparse{equal(\cellvalue,1)}%
    \ifnum\pgfmathresult=1\relax
      $\checkmark$%
    \fi
  },
  every node near coord/.append style={
    text=white,
    font=\scriptsize,
    anchor=center,
  },
  draw=white,
  line width=0.15pt,
] table[meta=z] {
x y z
1 1 0.346
2 1 1.000
3 1 1.000
4 1 0.998
5 1 0.998
6 1 0.998
7 1 1.000
8 1 1.000
9 1 1.000
10 1 1.000
1 2 0.358
2 2 1.000
3 2 0.998
4 2 1.000
5 2 1.000
6 2 1.000
7 2 1.000
8 2 1.000
9 2 1.000
10 2 1.000
1 3 0.358
2 3 1.000
3 3 1.000
4 3 1.000
5 3 1.000
6 3 0.992
7 3 0.988
8 3 0.976
9 3 0.942
10 3 0.886
1 4 0.344
2 4 1.000
3 4 0.990
4 4 0.936
5 4 0.790
6 4 0.738
7 4 0.732
8 4 0.724
9 4 0.686
10 4 0.594
1 5 0.348
2 5 0.986
3 5 0.838
4 5 0.738
5 5 0.732
6 5 0.722
7 5 0.650
8 5 0.510
9 5 0.302
10 5 0.146
1 6 0.014
2 6 0.234
3 6 0.386
4 6 0.506
5 6 0.702
6 6 0.586
7 6 0.690
8 6 0.628
9 6 0.776
10 6 0.660
};

\nextgroupplot[
  yticklabels={},
  xticklabels={},
  colorbar,
  colorbar style={
    y dir=normal,
    ytick={0,0.5,1},
    yticklabel style={font=\scriptsize},
    width=0.16cm,
    height=0.455\textwidth,
    at={(1.10,0.5)},
    anchor=west,
  },
]
\addplot[
  matrix plot*,
  mesh/cols=10,
  point meta=explicit,
  visualization depends on={\thisrow{z} \as \cellvalue},
  nodes near coords={%
    \pgfmathparse{equal(\cellvalue,1)}%
    \ifnum\pgfmathresult=1\relax
      $\checkmark$%
    \fi
  },
  every node near coord/.append style={
    text=white,
    font=\scriptsize,
    anchor=center,
  },
  draw=white,
  line width=0.15pt,
] table[meta=z] {
x y z
1 1 0.342
2 1 1.000
3 1 1.000
4 1 0.992
5 1 0.868
6 1 0.546
7 1 0.224
8 1 0.078
9 1 0.040
10 1 0.032
1 2 0.342
2 2 1.000
3 2 0.818
4 2 0.318
5 2 0.064
6 2 0.026
7 2 0.018
8 2 0.018
9 2 0.018
10 2 0.018
1 3 0.286
2 3 0.728
3 3 0.536
4 3 0.132
5 3 0.032
6 3 0.022
7 3 0.018
8 3 0.018
9 3 0.018
10 3 0.018
1 4 0.110
2 4 0.018
3 4 0.018
4 4 0.018
5 4 0.018
6 4 0.018
7 4 0.018
8 4 0.018
9 4 0.018
10 4 0.018
1 5 0.008
2 5 0.018
3 5 0.018
4 5 0.018
5 5 0.018
6 5 0.018
7 5 0.018
8 5 0.018
9 5 0.018
10 5 0.018
1 6 0.006
2 6 0.118
3 6 0.242
4 6 0.356
5 6 0.610
6 6 0.404
7 6 0.524
8 6 0.514
9 6 0.712
10 6 0.596
};

\nextgroupplot[
  ylabel={\texttt{float128}},
  xticklabels={0.01,0.02,0.03,0.04,0.05,0.06,0.07,0.08,0.09,0.10},
  xlabel={Aggregation-weight $A$},
]
\addplot[
  matrix plot*,
  mesh/cols=10,
  point meta=explicit,
  visualization depends on={\thisrow{z} \as \cellvalue},
  nodes near coords={%
    \pgfmathparse{equal(\cellvalue,1)}%
    \ifnum\pgfmathresult=1\relax
      $\checkmark$%
    \fi
  },
  every node near coord/.append style={
    text=white,
    font=\scriptsize,
    anchor=center,
  },
  draw=white,
  line width=0.15pt,
] table[meta=z] {
x y z
1 1 0.342
2 1 1.000
3 1 1.000
4 1 1.000
5 1 1.000
6 1 1.000
7 1 1.000
8 1 1.000
9 1 1.000
10 1 1.000
1 2 0.368
2 2 1.000
3 2 1.000
4 2 1.000
5 2 1.000
6 2 1.000
7 2 1.000
8 2 1.000
9 2 1.000
10 2 1.000
1 3 0.366
2 3 1.000
3 3 1.000
4 3 1.000
5 3 1.000
6 3 1.000
7 3 1.000
8 3 1.000
9 3 1.000
10 3 1.000
1 4 0.342
2 4 1.000
3 4 1.000
4 4 1.000
5 4 1.000
6 4 1.000
7 4 1.000
8 4 1.000
9 4 1.000
10 4 1.000
1 5 0.342
2 5 1.000
3 5 1.000
4 5 1.000
5 5 1.000
6 5 1.000
7 5 1.000
8 5 1.000
9 5 1.000
10 5 1.000
1 6 0.024
2 6 0.246
3 6 0.384
4 6 0.458
5 6 0.686
6 6 0.568
7 6 0.660
8 6 0.626
9 6 0.782
10 6 0.686
};

\nextgroupplot[
  yticklabels={},
  xticklabels={0.1,0.2,0.3,0.4,0.5,0.6,0.7,0.8,0.9,1.0},
  xlabel={Aggregation-weight $A$},
]
\addplot[
  matrix plot*,
  mesh/cols=10,
  point meta=explicit,
  visualization depends on={\thisrow{z} \as \cellvalue},
  nodes near coords={%
    \pgfmathparse{equal(\cellvalue,1)}%
    \ifnum\pgfmathresult=1\relax
      $\checkmark$%
    \fi
  },
  every node near coord/.append style={
    text=white,
    font=\scriptsize,
    anchor=center,
  },
  draw=white,
  line width=0.15pt,
] table[meta=z] {
x y z
1 1 0.344
2 1 1.000
3 1 1.000
4 1 1.000
5 1 1.000
6 1 1.000
7 1 1.000
8 1 1.000
9 1 1.000
10 1 1.000
1 2 0.358
2 2 1.000
3 2 1.000
4 2 1.000
5 2 1.000
6 2 1.000
7 2 1.000
8 2 1.000
9 2 1.000
10 2 1.000
1 3 0.362
2 3 1.000
3 3 1.000
4 3 1.000
5 3 1.000
6 3 0.992
7 3 0.988
8 3 0.986
9 3 0.978
10 3 0.948
1 4 0.342
2 4 1.000
3 4 0.994
4 4 0.974
5 4 0.844
6 4 0.752
7 4 0.734
8 4 0.732
9 4 0.728
10 4 0.726
1 5 0.346
2 5 0.992
3 5 0.934
4 5 0.756
5 5 0.734
6 5 0.730
7 5 0.728
8 5 0.724
9 5 0.708
10 5 0.634
1 6 0.016
2 6 0.246
3 6 0.396
4 6 0.512
5 6 0.676
6 6 0.568
7 6 0.682
8 6 0.632
9 6 0.778
10 6 0.650
};

\nextgroupplot[
  yticklabels={},
  xticklabels={1,2,3,4,5,6,7,8,9,10},
  xlabel={Aggregation-weight $A$},
]
\addplot[
  matrix plot*,
  mesh/cols=10,
  point meta=explicit,
  visualization depends on={\thisrow{z} \as \cellvalue},
  nodes near coords={%
    \pgfmathparse{equal(\cellvalue,1)}%
    \ifnum\pgfmathresult=1\relax
      $\checkmark$%
    \fi
  },
  every node near coord/.append style={
    text=white,
    font=\scriptsize,
    anchor=center,
  },
  draw=white,
  line width=0.15pt,
] table[meta=z] {
x y z
1 1 0.342
2 1 1.000
3 1 1.000
4 1 1.000
5 1 1.000
6 1 1.000
7 1 1.000
8 1 1.000
9 1 0.988
10 1 0.944
1 2 0.344
2 2 1.000
3 2 1.000
4 2 1.000
5 2 0.976
6 2 0.878
7 2 0.644
8 2 0.340
9 2 0.224
10 2 0.064
1 3 0.296
2 3 0.728
3 3 0.726
4 3 0.624
5 3 0.384
6 3 0.130
7 3 0.046
8 3 0.024
9 3 0.022
10 3 0.018
1 4 0.268
2 4 0.032
3 4 0.018
4 4 0.018
5 4 0.018
6 4 0.018
7 4 0.018
8 4 0.018
9 4 0.018
10 4 0.018
1 5 0.162
2 5 0.018
3 5 0.018
4 5 0.018
5 5 0.018
6 5 0.018
7 5 0.018
8 5 0.018
9 5 0.018
10 5 0.018
1 6 0.006
2 6 0.118
3 6 0.242
4 6 0.356
5 6 0.610
6 6 0.404
7 6 0.524
8 6 0.514
9 6 0.712
10 6 0.596
};
\end{groupplot}
\end{tikzpicture}
\caption{Success ratio for GNNs by floating-point precision (panel row),
fixed self-weight $C$ (panel column), activation (heatmap row), and
aggregation-weight $A$ (heatmap column). Checkmarks denote success ratio
$1$; all panels share the same color scale.}
\label{fig:all-precisions-fixed-c-heatmaps}
\end{figure*}

We would first like to add some clarity regarding our experimental setup.
When comparing embeddings produced by GNNs we compare without rounding.
That is to say, two embeddings are equal if they have the same binary value.
Graphs for which a given GNN produced values $\pm \infty$ or NaN
are automatically taken to be unsuccessful for that GNN.
In some systems 128-bit floating point numbers are unavailable.
We thus emulated float128 (IEEE 754 binary128) in software.

We present the extended sample of
Figure~\ref{fig:float64-fixed-c-heatmaps}
in Figure~\ref{fig:all-precisions-fixed-c-heatmaps}.
In particular we note that in the 128-bit setting
success ratio for sigmoid and tanh is exactly 1
when $A, C < 1$, and exactly 1 for sinh, cosh and exp when $A, C < 0.1$
(and as long as $\frac{C}{A} \not\in \mathbb{N}$).
This indicates that GNN errors in the float64 setting are due to imprecisions in float64 arithmetic.

It is interesting to analyse the effects that increased precision
has on
GNNs with different activation functions.
Firstly, we see no meaningful
gains for ReLU activated GNNs. This is to be expected
as the values computed by our ReLU GNNs are seemingly
well represented even in the float32 setting.
It is clear that increased precision correlates with better
performance for sigmoid and tanh activated GNNs.
Indeed, we see that in the float128 setting both sigmoid and tanh start to perform 
better for values above 1.
It is then interesting to contrast this with the seemingly weaker effect precision increase has on sinh, cosh and exp.
This phenomenon is consistent with 64- and 128-bit float
structure. Note that float64 reserves 11 bits for the exponent
and 53 bits of significand precision; float128 specifies 15 bits
for the exponent and 113 bits of significand precision.
Thus, float128 allows more accurate representation of
values close to 1, which is advantageous for
sigmoid and tanh. On the other hand,
sinh, cosh and exp are unbounded from above.
Thus, these activation functions coupled with larger weights exhaust 
the representable range of the float128 exponent.
See Table~\ref{tab:float32-overflow_rate} for how much 
of the failure rate is attributed to overflow.

Figure~\ref{fig:float128-medium-activation-heatmaps} presents
the full runs for sigmoid, tanh and sinh activated GNNs with $A, C \in \{ 0.1, 0.2, \dots, 1 \}$ in the 128-bit float setting. Notice
that in the case of sigmoid and tanh the success rate is 1
unless $\frac{C}{A} \in \mathbb{N}$.

%
\begin{figure*}[t]
\centering
\begin{tikzpicture}
\begin{groupplot}[
  group style={group size=3 by 1, horizontal sep=1.4em},
  width=0.235\textwidth,
  height=0.235\textwidth,
  scale only axis,
  xmin=0.5, xmax=10.5,
  ymin=0.5, ymax=10.5,
  y dir=normal,
  xtick={1,...,10},
  ytick={1,...,10},
  xticklabels={0.1,0.2,0.3,0.4,0.5,0.6,0.7,0.8,0.9,1.0},
  yticklabels={0.1,0.2,0.3,0.4,0.5,0.6,0.7,0.8,0.9,1.0},
  xlabel={Aggregation-weight $A$},
  tick label style={font=\scriptsize},
  x tick label style={font=\tiny,rotate=90,anchor=east},
  label style={font=\small},
  title style={font=\small},
  point meta min=0,
  point meta max=1,
  colormap={inferno}{
    rgb255(0cm)=(252,255,164);
    rgb255(1cm)=(246,215,70);
    rgb255(2cm)=(252,165,10);
    rgb255(3cm)=(243,120,25);
    rgb255(4cm)=(221,81,58);
    rgb255(5cm)=(188,55,84);
    rgb255(6cm)=(147,38,103);
    rgb255(7cm)=(106,23,110);
    rgb255(8cm)=(66,10,104);
    rgb255(9cm)=(22,11,57);
    rgb255(10cm)=(0,0,4)
  },
]

\nextgroupplot[
  title={sigmoid},
  ylabel={Self-weight $C$},
]
\addplot[
  matrix plot*,
  mesh/cols=10,
  point meta=explicit,
  visualization depends on={\thisrow{z} \as \cellvalue},
  nodes near coords={%
    \pgfmathparse{equal(\cellvalue,1)}%
    \ifnum\pgfmathresult=1\relax
      $\checkmark$%
    \fi
  },
  every node near coord/.append style={
    text=white,
    font=\scriptsize,
    anchor=center,
  },
  draw=white,
  line width=0.15pt,
] table[meta=z] {
x y z
1 1 0.344
2 1 1.000
3 1 1.000
4 1 1.000
5 1 1.000
6 1 1.000
7 1 1.000
8 1 1.000
9 1 1.000
10 1 1.000
1 2 0.992
2 2 0.342
3 2 1.000
4 2 1.000
5 2 1.000
6 2 1.000
7 2 1.000
8 2 1.000
9 2 1.000
10 2 1.000
1 3 1.000
2 3 1.000
3 3 0.342
4 3 1.000
5 3 1.000
6 3 1.000
7 3 1.000
8 3 1.000
9 3 1.000
10 3 1.000
1 4 1.000
2 4 0.992
3 4 1.000
4 4 0.342
5 4 1.000
6 4 1.000
7 4 1.000
8 4 1.000
9 4 1.000
10 4 1.000
1 5 1.000
2 5 1.000
3 5 1.000
4 5 1.000
5 5 0.344
6 5 1.000
7 5 1.000
8 5 1.000
9 5 1.000
10 5 1.000
1 6 1.000
2 6 1.000
3 6 0.992
4 6 1.000
5 6 1.000
6 6 0.342
7 6 1.000
8 6 1.000
9 6 1.000
10 6 1.000
1 7 1.000
2 7 1.000
3 7 1.000
4 7 1.000
5 7 1.000
6 7 1.000
7 7 0.346
8 7 1.000
9 7 1.000
10 7 1.000
1 8 1.000
2 8 1.000
3 8 1.000
4 8 0.992
5 8 1.000
6 8 1.000
7 8 1.000
8 8 0.346
9 8 1.000
10 8 1.000
1 9 1.000
2 9 1.000
3 9 1.000
4 9 1.000
5 9 1.000
6 9 1.000
7 9 1.000
8 9 1.000
9 9 0.346
10 9 1.000
1 10 1.000
2 10 1.000
3 10 1.000
4 10 1.000
5 10 0.992
6 10 1.000
7 10 1.000
8 10 1.000
9 10 1.000
10 10 0.342
};

\nextgroupplot[
  title={tanh},
  yticklabels={},
]
\addplot[
  matrix plot*,
  mesh/cols=10,
  point meta=explicit,
  visualization depends on={\thisrow{z} \as \cellvalue},
  nodes near coords={%
    \pgfmathparse{equal(\cellvalue,1)}%
    \ifnum\pgfmathresult=1\relax
      $\checkmark$%
    \fi
  },
  every node near coord/.append style={
    text=white,
    font=\scriptsize,
    anchor=center,
  },
  draw=white,
  line width=0.15pt,
] table[meta=z] {
x y z
1 1 0.358
2 1 1.000
3 1 1.000
4 1 1.000
5 1 1.000
6 1 1.000
7 1 1.000
8 1 1.000
9 1 1.000
10 1 1.000
1 2 0.992
2 2 0.356
3 2 1.000
4 2 1.000
5 2 1.000
6 2 1.000
7 2 1.000
8 2 1.000
9 2 1.000
10 2 1.000
1 3 1.000
2 3 1.000
3 3 0.376
4 3 1.000
5 3 1.000
6 3 1.000
7 3 1.000
8 3 1.000
9 3 1.000
10 3 1.000
1 4 1.000
2 4 0.992
3 4 1.000
4 4 0.362
5 4 1.000
6 4 1.000
7 4 1.000
8 4 1.000
9 4 1.000
10 4 1.000
1 5 1.000
2 5 1.000
3 5 1.000
4 5 1.000
5 5 0.364
6 5 1.000
7 5 1.000
8 5 1.000
9 5 1.000
10 5 1.000
1 6 1.000
2 6 1.000
3 6 0.992
4 6 1.000
5 6 1.000
6 6 0.364
7 6 1.000
8 6 1.000
9 6 1.000
10 6 1.000
1 7 1.000
2 7 1.000
3 7 1.000
4 7 1.000
5 7 1.000
6 7 1.000
7 7 0.356
8 7 1.000
9 7 1.000
10 7 1.000
1 8 1.000
2 8 1.000
3 8 1.000
4 8 0.992
5 8 1.000
6 8 1.000
7 8 1.000
8 8 0.348
9 8 1.000
10 8 1.000
1 9 1.000
2 9 1.000
3 9 1.000
4 9 1.000
5 9 1.000
6 9 1.000
7 9 1.000
8 9 1.000
9 9 0.348
10 9 1.000
1 10 1.000
2 10 1.000
3 10 1.000
4 10 1.000
5 10 0.992
6 10 1.000
7 10 1.000
8 10 1.000
9 10 1.000
10 10 0.344
};

\nextgroupplot[
  title={sinh},
  yticklabels={},
  colorbar,
  colorbar style={
    y dir=normal,
    ytick={0,0.5,1},
    yticklabel style={font=\scriptsize},
    width=0.16cm,
  },
]
\addplot[
  matrix plot*,
  mesh/cols=10,
  point meta=explicit,
  visualization depends on={\thisrow{z} \as \cellvalue},
  nodes near coords={%
    \pgfmathparse{equal(\cellvalue,1)}%
    \ifnum\pgfmathresult=1\relax
      $\checkmark$%
    \fi
  },
  every node near coord/.append style={
    text=white,
    font=\scriptsize,
    anchor=center,
  },
  draw=white,
  line width=0.15pt,
] table[meta=z] {
x y z
1 1 0.362
2 1 1.000
3 1 1.000
4 1 1.000
5 1 1.000
6 1 0.992
7 1 0.988
8 1 0.986
9 1 0.978
10 1 0.948
1 2 0.992
2 2 0.358
3 2 1.000
4 2 1.000
5 2 1.000
6 2 0.992
7 2 0.988
8 2 0.986
9 2 0.976
10 2 0.944
1 3 1.000
2 3 1.000
3 3 0.358
4 3 1.000
5 3 1.000
6 3 0.992
7 3 0.988
8 3 0.984
9 3 0.976
10 3 0.936
1 4 1.000
2 4 0.992
3 4 1.000
4 4 0.366
5 4 0.998
6 4 0.992
7 4 0.986
8 4 0.982
9 4 0.970
10 4 0.928
1 5 1.000
2 5 1.000
3 5 1.000
4 5 1.000
5 5 0.354
6 5 0.990
7 5 0.986
8 5 0.980
9 5 0.960
10 5 0.922
1 6 1.000
2 6 1.000
3 6 0.992
4 6 1.000
5 6 0.992
6 6 0.352
7 6 0.984
8 6 0.976
9 6 0.950
10 6 0.898
1 7 1.000
2 7 1.000
3 7 1.000
4 7 0.998
5 7 0.992
6 7 0.986
7 7 0.354
8 7 0.974
9 7 0.944
10 7 0.878
1 8 1.000
2 8 1.000
3 8 1.000
4 8 0.988
5 8 0.990
6 8 0.986
7 8 0.982
8 8 0.346
9 8 0.930
10 8 0.868
1 9 1.000
2 9 1.000
3 9 1.000
4 9 0.994
5 9 0.986
6 9 0.984
7 9 0.976
8 9 0.956
9 9 0.340
10 9 0.848
1 10 1.000
2 10 1.000
3 10 0.998
4 10 0.992
5 10 0.978
6 10 0.982
7 10 0.974
8 10 0.948
9 10 0.884
10 10 0.296
};
\end{groupplot}
\end{tikzpicture}
\caption{Success ratio in float128 arithmetic for GNNs with medium
self-weight $C$ (heatmap row), medium aggregation-weight $A$ (heatmap column),
and the sigmoid, tanh, and sinh activations (panels). Checkmarks denote
success ratio $1$; all panels share the same color scale.}
\label{fig:float128-medium-activation-heatmaps}
\end{figure*}
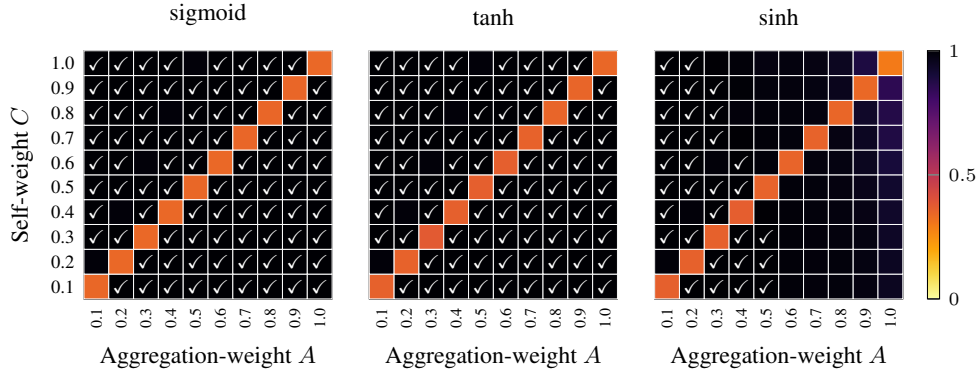

\begin{table}[h]
\centering
\small
\setlength{\tabcolsep}{4pt}
\begin{tabular}{lccccc}
\toprule
 & sigmoid & tanh & sinh & cosh & exp \\
\midrule
Success ratio           & 0.982 & 0.998 & 0.974 & 0.634 & 0.116 \\
Non-finite failure rate & 0.000 & 0.000 & 0.022 & 0.366 & 0.884 \\
\bottomrule
\end{tabular}
\caption{Float32 results at self-weight $C=0.1$ and aggregation-weight
$A=0.7$. A non-finite failure means that at least one final node embedding is
NaN, $+\infty$, or $-\infty$; these outcomes are combined into one category.}
\label{tab:float32-overflow_rate}
\end{table}

\end{document}